%% file: main.tex
\documentclass{article}

\usepackage[accepted]{icml2021}

\usepackage{amsmath,amssymb,amsfonts,amsthm}
\usepackage{algorithm,algorithmic}
\usepackage{optidef}
\usepackage{nicefrac,siunitx}
\usepackage{microtype}
\usepackage{graphicx}
\usepackage{subfigure}
\usepackage{booktabs}
\usepackage{multirow}

\usepackage{enumitem}

\usepackage{catchfile}
\newcommand\loaddata[1]{\CatchFileDef\loadeddata{#1}{\endlinechar=-1}}

\usepackage{tikz}
\usetikzlibrary{arrows.meta}
\usetikzlibrary{positioning}
\usetikzlibrary{backgrounds}
\usetikzlibrary{matrix}
\usetikzlibrary{decorations.markings}

\usepackage{pgfplots}
\usepgfplotslibrary{groupplots}
\usepgfplotslibrary{colormaps}
\usepgfplotslibrary{colorbrewer}
\pgfplotsset{compat=1.3}

\usetikzlibrary{external}
\usepackage{doi} 

\usepackage{hyperref}
\hypersetup{
  colorlinks = true,
  allcolors = blue,
}
\usepackage[capitalise,nameinlink]{cleveref}

\newcommand{\parheading}[1]{\vspace{1mm}\noindent\textit{#1}}

\newtheorem{prop}{Proposition}
\newtheorem{theorem}{Theorem}
\newtheorem{lemma}{Lemma}

\newtheoremstyle{myremark}
{\topsep} 
{\topsep} 
{\normalfont} 
{} 
{\bfseries} 
{.} 
{5pt plus 1pt minus 1pt} 
{\thmname{#1}\thmnumber{ #2}\thmnote{ (#3)}} 

\theoremstyle{myremark}

\theoremstyle{definition}
\newtheorem{defn}{Definition}
\newtheorem{property}{Property}

\crefmultiformat{prop}{Propositions~#2#1#3}{ and~#2#1#3}{, #2#1#3}{ and~#2#1#3}
\crefmultiformat{theorem}{Theorems~#2#1#3}{ and~#2#1#3}{, #2#1#3}{ and~#2#1#3}
\crefmultiformat{lemma}{Lemmas~#2#1#3}{ and~#2#1#3}{, #2#1#3}{ and~#2#1#3}
\crefmultiformat{coro}{Corollaries~#2#1#3}{ and~#2#1#3}{, #2#1#3}{ and~#2#1#3}
\crefmultiformat{problem}{Problems~#2#1#3}{ and~#2#1#3}{, #2#1#3}{ and~#2#1#3}
\crefmultiformat{assump}{Assumptions~#2#1#3}{ and~#2#1#3}{, #2#1#3}{ and~#2#1#3}

\crefmultiformat{defn}{Definitions~#2#1#3}{ and~#2#1#3}{, #2#1#3}{ and~#2#1#3}
\crefrangeformat{defn}{Definitions.~#3#1#4,#5#2#6}

\crefmultiformat{property}{Properties~#2#1#3}{ and~#2#1#3}{, #2#1#3}{ and~#2#1#3}
\crefrangeformat{property}{Properties~#3#1#4,#5#2#6}

\DeclareMathOperator*{\argmax}{arg\,max}

\newcommand{\SCN}{\mathrm{SCN}}
\newcommand{\SCONE}{\mathrm{SCoNe}}

\newcommand{\ReLU}{\mathrm{ReLU}}

\newcommand{\image}{\mathrm{im}}
\newcommand{\kernel}{\mathrm{ker}}

\newcommand{\ie}{\textit{i.e.}}
\newcommand{\eg}{\textit{e.g.}}
\newcommand{\cf}[1]{[\textit{cf.}~#1]}

\newcommand{\bigO}{\mathcal{O}}

\newcommand{\R}{\mathbb{R}}
\newcommand{\Rn}[1]{\mathbb{R}^{#1}}
\newcommand{\Rnn}[1]{\mathbb{R}^{#1\times #1}}
\newcommand{\Rnp}[2]{\mathbb{R}^{#1\times #2}}

\newcommand{\bA}{\mathbf{A}}
\newcommand{\bB}{\mathbf{B}}

\newcommand{\bD}{\mathbf{D}}

\newcommand{\bI}{\mathbf{I}}
\newcommand{\bL}{\mathbf{L}}
\newcommand{\bP}{\mathbf{P}}

\newcommand{\bW}{\mathbf{W}}

\newcommand{\cC}{\mathcal{C}}
\newcommand{\cD}{\mathcal{D}}
\newcommand{\cE}{\mathcal{E}}

\newcommand{\cN}{\mathcal{N}}
\newcommand{\cP}{\mathcal{P}}

\newcommand{\cV}{\mathcal{V}}
\newcommand{\cX}{\mathcal{X}}

\newcommand{\bc}{\mathbf{c}}
\newcommand{\bd}{\mathbf{d}}

\newcommand{\bw}{\mathbf{w}}
\newcommand{\bx}{\mathbf{x}}
\newcommand{\by}{\mathbf{y}}
\newcommand{\bz}{\mathbf{z}}

\usepackage{color}

\icmltitlerunning{Principled Simplicial Neural Networks for Trajectory Prediction}

\begin{document}

\twocolumn[
\icmltitle{Principled Simplicial Neural Networks for Trajectory Prediction}

\icmlsetsymbol{equal}{*}

\begin{icmlauthorlist}
  \icmlauthor{T.~Mitchell~Roddenberry}{equal,ece}
  \icmlauthor{Nicholas~Glaze}{equal,ece}
  \icmlauthor{Santiago~Segarra}{ece}
\end{icmlauthorlist}

\icmlaffiliation{ece}{Department of Electrical and Computer Engineering, Rice University, Houston, Texas, USA}

\icmlcorrespondingauthor{TMR}{\href{mailto:mitch@rice.edu}{mitch@rice.edu}}
\icmlcorrespondingauthor{NG}{\href{mailto:nkg2@rice.edu}{nkg2@rice.edu}}
\icmlcorrespondingauthor{SS}{\href{mailto:segarra@rice.edu}{segarra@rice.edu}}

\icmlkeywords{}

\vskip 0.3in
]

\printAffiliationsAndNotice{\icmlEqualContribution}

\begin{abstract}
We consider the construction of neural network architectures for data on simplicial complexes. 
In studying maps on the chain complex of a simplicial complex, we define three desirable properties of a simplicial neural network architecture: namely, permutation equivariance, orientation equivariance, and simplicial awareness. 
The first two properties respectively account for the fact that the indexing and orientations of simplices in a simplicial complex are arbitrary. 
The last property requires the output of the neural network to depend on the entire simplicial complex and not on a subset of its dimensions. 
Based on these properties, we propose a simple convolutional architecture, rooted in tools from algebraic topology, for the problem of trajectory prediction, and show that it obeys all three of these properties when an odd, nonlinear activation function is used. 
We then demonstrate the effectiveness of this architecture in extrapolating trajectories on synthetic and real datasets, with particular emphasis on the gains in generalizability to unseen trajectories.
\end{abstract}

\section{Introduction}

Graph neural networks have shown great promise in combining the representational power of neural networks with the structure imparted by a graph.
In essence, graph neural networks compute a sequence of node representations by aggregating information at each node from its neighbors and itself, then applying a nonlinear transformation.
Using variations on this approach, many architectures have exhibited state-of-the-art performance in tasks including node classification~\citep{Velickovic:2018}, link prediction~\citep{Zhang:2018}, and graph classification~\citep{Hamilton:2017}.
Indeed, the strength of graph neural networks lies in their ability to incorporate arbitrary pairwise relational structures in their computations.

However, not all data is adequately expressed in terms of pairwise relationships, nor is it strictly supported on the nodes of a graph.
Interactions in a social network, for instance, do not solely occur in a pairwise fashion but also among larger groups of people.
This warrants a higher-order model in order to represent rich and complex datasets.
One such higher-order model is the (abstract) simplicial complex, which describes relational structures that are closed under restriction: if three people are all friends together, then each pair of people in that group are also friends.
We can understand data supported on simplicial complexes using tools from algebraic topology~\citep{Hatcher:2005,Carlsson:2009,Ghrist:2014}; 
in particular, we can analyze data supported on the edges and higher-order structures using the spectrum of certain linear operators on the simplicial complex.
This is analogous to the use of spectral graph theory~\citep{Chung:1997} to understand the smoothness of data supported on the nodes of a graph~\citep{Shuman:2013}.

\subsection{Contribution}

We study the extension of graph convolutional network architectures to process data supported on simplicial complexes.
After establishing appropriate operators for such data in \cref{sec:background}, our first contribution is to define in \cref{sec:admissible} a notion of admissibility, in terms of three reasonable properties that we require of neural networks for simplicial complexes.
We then focus on the problem of trajectory prediction over a simplicial complex, proposing a simple convolutional neural architecture in \cref{sec:scone}, which we design with admissibility in mind.
In particular, we show that the activation functions of the proposed architecture must be odd and nonlinear in order to satisfy our requirements for admissibility.
We empirically illustrate how admissibility yields better generalizability to unseen trajectories in \cref{sec:experiments}.

\section{Related work}

\subsection{Graph Neural Networks}

Graph neural networks extend the success of convolutional neural networks for Euclidean data to the graph domain, adapting the weight-sharing of convolutional networks in a way that reflects the underlying graph structure.
Effectively, graph neural networks interleave nonlinear activation functions with diffusion operators dictated by the graph structure, such as the adjacency or Laplacian matrices.
\citet{Bruna:2014} developed graph neural networks in the Laplacian spectral domain, which was further simplified by \citet{Defferrard:2016}, who expressed the diffusion operator at each layer as a low-order Chebyshev polynomial in order to improve scalability.
\citet{Kipf:2017} reduced this further, employing a first-order diffusion operator at each layer expressed in terms of the graph Laplacian.
We refer the reader to \citet{Wu:2020} for a survey of graph neural network architectures and applications.

\subsection{Signal Processing on Simplicial Complexes}

Aiming to extend the field of signal processing on graphs~\citep{Shuman:2013}, recent works have used tools from algebraic topology to understand data supported on simplicial complexes.
Rooted in discrete calculus on graphs, and in particular combinatorial Hodge theory~\citep{Jiang:2011,Lim:2020}, recent works have considered the use of the Hodge Laplacian for the analysis of flows on graphs~\citep{Schaub:2018,Barbarossa:2018,Barbarossa:2020a,Barbarossa:2020b}.
This line of research was distinct from existing developments in graph signal processing, in that it handled flows defined with respect to an arbitrary orientation assigned to each edge, much like the analysis of current in an electrical circuit.
One application of this perspective was studied by \citet{Jia:2019}, where the problem of flow interpolation was cast as an optimization problem, minimizing the quadratic form of the Hodge Laplacian subject to observation constraints.

\subsection{Simplicial Neural Networks}

The first application of discrete Hodge theory to the design of neural networks was proposed by~\citet{Roddenberry:2019}.
This work analyzed edge-flow data, focusing on the problems of flow interpolation and source localization on simplicial complexes, while drawing attention to the value of permutation and orientation equivariance.
Since then, other works~\citep{Ebli:2020,Bunch:2020} have also relied on discrete Hodge theory to propose convolutional neural architectures for data supported on simplicial complexes.
However, these works do not address important notions of orientation, which we discuss in~\cref{sec:admissible}.
Instead of solely proposing a \emph{specific} neural architecture, our current work develops a principled \emph{framework} to construct generalizable simplicial neural network architectures, proposes a simple architecture following those principles, and illustrates its advantages compared with competing approaches.

\section{Background}\label{sec:background}

\subsection{Simplicial Complexes}

An \emph{abstract simplicial complex} is a set $\cX$ of finite subsets of another set $\cV$ that is closed under restriction, \ie{} for all $\sigma\in\cX$, if $\sigma'\subseteq\sigma$, we have $\sigma'\in\cX$.
Each such element of $\cX$ is called a \emph{simplex}: in particular, if $|\sigma|=k+1$, we call $\sigma$ a \emph{$k$-simplex}.
For a $k$-simplex $\sigma$, its \emph{faces} are all of the $(k-1)$-simplices that are also subsets of $\sigma$, while its \emph{cofaces} are all $(k+1)$-simplices that have $\sigma$ as a face.

Grounding these definitions in our intuition for graphs, we refer to the elements of $\cV$ as \emph{nodes}, or equivalently the $0$-simplices of $\cX$.
We refer to the $1$-simplices of $\cX$ as \emph{edges}, and the $2$-simplices as \emph{triangles}, corresponding to ``filled-in triangles'' in a departure from classical graphs.
The edges, then, are faces of the triangles, with the nodes being faces of the edges.
We refer to higher-order simplices by their order: $k$-simplices.
For convenience, we use the notation $\cX_k$ to refer to the collection of $k$-simplices of $\cX$, \eg{}, $\cX_0=\cV$.
The \emph{dimension} of a simplicial complex $\cX$ is the maximal $k$ such that $\cX_k$ is nonempty.
The \emph{$k$-skeleton} of a simplicial complex is the union $\bigcup_{\ell=0}^k\cX_\ell$.
We will later find it convenient to refer to the \emph{neighborhood} of a node $i$: for a simplicial complex $\cX$, denote by $\cN(i)$ the set of all $j\in\cX_0$ such that $\{i,j\}\in\cX_1$, \ie{}, the set of nodes connected to $i$ by an edge.
Or, using our defined terminology for simplicial complexes, the neighborhood of a node $i$ consists of the faces of the cofaces of $i$, excluding $i$ itself.

We arbitrarily endow each simplex of $\cX$ with an \emph{orientation}.
An orientation can be thought of as a chosen ordering of the constituent elements of a simplex, modulo even permutations.
That is, for a $k$-simplex $\sigma=\{i_0,i_1\ldots,i_k\}\subseteq\cV$, an orientation of $\sigma$ would be $[i_0,i_1,\ldots,i_k]$.
Performing an even permutation of this (\eg{}, $[i_1,i_2,i_0,\ldots,i_k]$) leads to an equivalent orientation.
For convenience, we label the nodes with the non-negative integers, and let the chosen orientations of all simplices be given by the ordering induced by the node labels.

\subsection{Boundary Operators and Hodge Laplacians}

Denote by $\cC_k$ the vector space with the oriented $k$-simplices of $\cX$ as a canonical orthonormal basis, defined over the field of real numbers.
We define notions of matrix multiplication and linear operators, \eg{}, diagonal matrices and permutation matrices, with respect to this basis.
Each element of $\cC_k$ is called a \emph{$k$-chain} and is subject to all the properties of a vector.
In particular, multiplying a $k$-chain by $+1$ is idempotent, while multiplication by $-1$ reverses its orientation (by reversing the orientations of its basis vectors).
For instance, let $[i_0,i_1]\in\cC_1$.
We then have $[i_0,i_1]=-[i_1,i_0]$.
Using this basis of oriented simplices, we endow the vector spaces $\cC_k$ with the usual properties of finite-dimensional real vector spaces such as inner products and linear maps expressed as real matrices.

A particular set of linear maps between chains is the set of \emph{boundary operators\footnote{We ignore the boundary operator $\partial_0$, which maps $\cC_0\to\{0\}$.}} $\{\partial_k\}_{k=1}^K$, where $K$ is the highest order of any simplex in $\cX$.
For an oriented $k$-simplex $\sigma=[i_0,i_1,\ldots,i_k]$, the boundary operator $\partial_k:\cC_k\to\cC_{k-1}$ is defined as
\begin{equation}\label{eq:boundary}
  \partial_k\sigma = \sum_{j=0}^k(-1)^j[i_0,\ldots,i_{j-1},i_{j+1},\ldots,i_k].
\end{equation}
That is, the boundary operator takes an ordered, alternating sum of the faces that form the boundary of $\sigma$.
The collection of these vector spaces coupled with the boundary maps forms what is known as a \emph{chain complex}.
Indeed, to represent a simplicial complex it is sufficient to use its boundary maps.
An important result in algebraic topology relates these boundary operators to each other.
\begin{lemma}\label{lem:fun-thm-alg-top}
  The boundary operator squared is null.
  That is, for all $k$, $\partial_{k-1}\circ\partial_k=0$.
\end{lemma}
Additionally, the boundary operator induces a \emph{co-boundary operator}, which is the adjoint of $\partial$, \ie{} $\partial_k^\top$.
Based on the boundary and coboundary operators, we define the \emph{$k^\mathrm{th}$ Hodge Laplacian} as
\begin{equation}
  \Delta_k=\partial_k^\top\partial_k+\partial_{k+1}\partial_{k+1}^\top.
\end{equation}
Since the linear map $\partial_1$ can be represented via the signed node-edge incidence matrix, one can check that $\Delta_0$ recovers the graph Laplacian (see supplementary material for details).
More generally, $\Delta_k:\cC_k\to\cC_k$ is a linear operator on the space of $k$-chains for each $k$.
The \emph{Hodge Decomposition} allows us to view the vector space $\cC_k$ in terms of the boundary maps and the Hodge Laplacian.
\begin{theorem}[Hodge Decomposition]\label{thm:hodge}
  For a simplicial complex $\cX$ with boundary maps $\partial=\{\partial_k\}_{k=1}^K$ we have that
  \begin{equation}
    \cC_k=\image(\partial_{k+1})\oplus\image(\partial_k^\top)\oplus\kernel(\Delta_k)
  \end{equation}
  for all $k$, where $\oplus$ represents the (orthogonal) direct sum.
\end{theorem}
This result is particularly pleasing, as it gives us a convenient representation of a $k$-chain in terms of the ``upper'' and ``lower'' incidence structures of the simplices on which it is supported, as defined by its boundary operators.
Moreover, we have that $\dim\kernel(\Delta_k)=\beta_k$, where $\beta_k$ is the $k^\mathrm{th}$ \emph{Betti number}, which counts the number of ``$k$-dimensional holes'' in $\cX$~\citep{Carlsson:2009}.

Of particular interest in this work is the Hodge Decomposition of the space $\cC_1$, which models ``flows'' on simplicial complexes, as studied in detail by \citet{Schaub:2020}.
Indeed, $1$-chains on a simplicial complex are a natural way to discretize a continuous vector field~\citep{Barbarossa:2020b}, or to model the flow of traffic in a road network~\citep{Jia:2019,Roddenberry:2019}.
That is, we consider the decomposition $\cC_1=\image(\partial_2)\oplus\image(\partial_1^\top)\oplus\kernel(\Delta_1)$.

First, we note that $\image(\partial_2)$ corresponds to $1$-chains that are \emph{curly} with respect to the triangles in a simplicial complex: 
that is, such $1$-chains consist of flows around the boundary of triangles.

Next, we observe that $\image(\partial_1^\top)$ corresponds to $1$-chains induced by node \emph{gradients}.
Precisely, $1$-chains in the image of $\partial_1^\top$ are determined by a set of scalar values on the nodes, whose local differences dictate the coefficients of the $1$-chain, analogously to vector fields in Euclidean space induced by the gradient of a scalar field.

Finally, the subspace of $\cC_1$ determined by $\kernel(\Delta_1)$ consists of $1$-chains that are neither curly nor gradient: we call such chains \emph{harmonic}.
Harmonic chains are $1$-chains with the property that the sum of the coefficients around any triangle is zero, while the sum of the flow coefficients incident to any node is also zero.
This subspace is of particular interest, since it captures a natural notion of a smooth, conservative flow, as leveraged by \citet{Ghosh:2018,Jia:2019,Schaub:2018,Schaub:2020,Barbarossa:2018,Barbarossa:2020a}.
We refer the reader to these works for more in-depth discussion of modeling flows with $1$-chains and the relevant applications of discrete Hodge theory, as well as to the supplementary material for a more in-depth discussion of \cref{thm:hodge}.

\section{Admissible Neural Architectures}\label{sec:admissible}

We define three desirable properties of a graph neural network acting on chains supported by a simplicial complex.
These properties will be later leveraged to construct our proposed architecture for the task of trajectory prediction.
Throughout, let $\SCN_{\bW,\partial}:\cC_j\to\cC_\ell$ denote a neural network architecture acting on input data $\bc_j\in\cC_j$ and whose output is a chain of possibly different order, where the neural network is parameterized by a collection of weights $\bW$ and boundary operators $\partial$.

\subsection{Permutation Equivariance}
%
When representing graphs and related structures with matrices, a key property is \emph{permutation equivariance}.
For instance, if we take a graph with adjacency matrix $\bA$, then multiplying $\bA$ from both sides by a permutation matrix, \ie{} $\bP\bA\bP^\top$, corresponds to relabeling the nodes in the original graph. 
In this way, $\bA$ and $\bP\bA\bP^\top$ are alternative matrix representations of the \emph{same} graph.
Therefore, in order to ensure that an operation does not depend on the specific (arbitrary) node labeling, this operation must be impervious to the application of a permutation matrix.
More formally, and for simplicial complexes in general, we define permutation equivariance as follows.
\begin{property}[Permutation Equivariance]\label{property:perm-equi}
  Let $\cX$ be a simplicial complex with boundary maps $\partial=\{\partial_k\}_{k=1}^K$.
  Let $\cP=\{\bP_k\}_{k=0}^K$ be a collection of permutation matrices matching the dimensions of $\{\cC_k\}_{k=0}^K$, \ie{} $\bP_k\in\Rnn{|\cX_k|}$, and define $[\cP\partial]_k:=\bP_{k-1}\partial_k\bP_k^\top$.
  We say that $\SCN$ satisfies \emph{permutation equivariance} if for any such $\cP$, we have that
  \begin{equation}\label{eq:permutation_equivariance}
    \SCN_{\bW,\partial}(\bc_j)=\bP_\ell\SCN_{\bW,\cP\partial}(\bP_j\bc_j).
  \end{equation}
\end{property}

The above expression guarantees that if we relabel the simplicial complex and apply a neural network, the output is a relabeled version of the output that we would have obtained by applying the neural network prior to relabeling.

\subsection{Orientation Equivariance}
%
In defining the boundary operators $\partial$, we choose an orientation for each simplex in $\cX$.
The choice of this boundary is arbitrary, and only serves to meaningfully represent boundary operations and useful signals on simplicial complexes.
Similar to the arbitrary choice of ordering in the matrix representation motivating permutation equivariance, we also require an architecture to be insensitive to the chosen orientations.
Recalling that reversing the orientation of a $k$-chain is equivalent to multiplying it by $-1$, we define orientation equivariance as follows.

\begin{property}[Orientation Equivariance]\label{property:orient-equi}
  Let $\cX$ be a simplicial complex with boundary maps $\partial=\{\partial_k\}_{k=1}^K$.
  Let $\cD=\{\bD_k\}_{k=0}^K$ be a collection of diagonal matrices with values taking $\pm 1$ with the condition that $\bD_0=\bI$, and matching the dimensions of $\{\cC_k\}_{k=0}^K$, \ie{} $\bD_k\in\Rnn{|\cX_k|}$, and define $[\cD\partial]_k:=\bD_{k-1}\partial_k\bD_k$.
  We say that $\SCN$ satisfies \emph{orientation equivariance} if for any given $\cX,\cD,\bW,\bc_j$, we have that
  \begin{equation}\label{eq:orientation_equivariance}
    \SCN_{\bW,\partial}(\bc_j)=\bD_\ell\SCN_{\bW,\cD\partial}(\bD_j\bc_j).
  \end{equation}
\end{property}

The intuition in~\eqref{eq:orientation_equivariance} is analogous to that in~\eqref{eq:permutation_equivariance}, but focuses on changes in orientation as opposed to relabeling. Also, notice that orientation is only defined for simplices of order at least $1$, \ie{}, edges and higher.
This is due to the simple fact that the nodes of a simplicial complex naturally do not have an orientation: there is only one permutation of a singleton set.
For this reason, we require $\bD_0=\bI$ in~\cref{property:orient-equi}.

\subsection{Simplicial Awareness}
%
The notions of permutation and orientation equivariance are fairly intuitive, corresponding to very common constructs in the analysis of graph-structured data, as well as graph neural networks in particular.
However, the higher-order structures in simplicial complexes motivate architectures for data supported on simplices of different order, and regularized by hierarchically organized structure.
To this end, we define the notion of \emph{simplicial awareness}, which enforces dependence of an architecture's output on all of the boundary operators.

\begin{property}[Simplicial Awareness]\label{property:simp-aware}
  Let $\SCN_{\bW,\partial}:\cC_j\to\cC_\ell$ and select some integer $k>0$ such that $k \neq j$ and $k \neq \ell$.
  Suppose there exists simplicial complexes $\cX$ and $\cX'$ such that
  $\cX_0=\cX'_0=\cV,
  \cX_j=\cX'_j,
  \cX_\ell=\cX'_\ell,
  \cX_k\neq\cX'_k$.
  Denote the respective boundary operators of $\cX$ and $\cX'$ by $\partial$ and $\partial'$.
  If there exists $\bc_j\in\cC_j$ and weight parameters $\bW$ where
  \begin{equation}
    \SCN_{\bW,\partial}(\bc_j)\neq\SCN_{\bW,\partial'}(\bc_j),
  \end{equation}
  we say that $\SCN$ satisfies \emph{simplicial awareness of order~$k$}.
  Moreover, for the set of simplicial complexes of dimension at most $K$, if the above is satisfied for all $k\leq K$, then we simply say that $\SCN$ satisfies \emph{simplicial awareness}.
\end{property}

Put simply, simplicial awareness of order $k$ indicates that an architecture is not independent of the $k$-simplices in the underlying simplicial complex.
For example, consider a simplicial complex $\cX$ composed of nodes, edges, and triangles, and $\SCN_{\bW,\partial}:\cC_0 \to \cC_0$.
One can envision $\SCN_{\bW,\partial}$ in the form of a standard graph neural network by ignoring the triangles, but this would violate simplicial awareness of order $2$.

\subsection{Admissibility}
%
We define a notion of \emph{admissibility} that we use to guide our design of neural networks acting on chain complexes.
\begin{defn}\label{def:admissibility}
  An architecture is \emph{admissible} if it satisfies \hyperref[property:perm-equi]{permutation equivariance}, \hyperref[property:orient-equi]{orientation equivariance}, and \hyperref[property:simp-aware]{simplicial awareness}.
\end{defn}
We define admissibility largely as a suggestion: of course, neural network architectures need not satisfy these three properties.
However, much like the permutation equivariance of graph convolutional networks, enforcing the corresponding symmetries in a simplicial neural network ensures that the design is not subject to the user's choice of permutation.
Enforcing the property of permutation equivariance enables graph neural networks trained on small graphs to generalize well to larger graphs, since it is not dependent on a set of hand-selected node labels~\citep{Hamilton:2020}.
Indeed, by enforcing meaningful symmetries under group actions in a domain, neural architectures learn more efficient, generalizable representations~\citep{Cohen:2016}.

For the purposes of a simplicial complex, the same logic applies: the neural architecture itself should reflect the symmetries and invariances of the underlying domain, in order to promote generalization to unseen structures in a way that is not subject to design by the user.
The motivation for this in the setting of simplicial complexes is highlighted by the property of orientation equivariance: since the orientation of simplices is arbitrary, we do not want to train an architecture that is dependent on the chosen orientation of the training data, since there is little hope of it working well on unseen structures without careful user-selected orientations.

Finally, simplicial awareness enforces a minimum representational capacity: if a neural architecture is incapable of incorporating information from certain structural features of the simplicial complex, one can construct vastly different datasets with the exact same output, \eg{}, a simplicial complex that is triangle-dense, compared to its $1$-skeleton.

\section{Trajectory Prediction with SCoNe}\label{sec:scone}

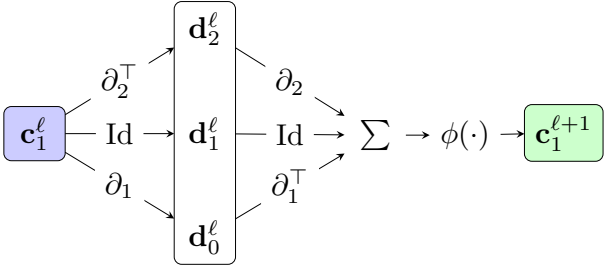
\begin{figure}
  \centering
  \resizebox{\columnwidth}{!}{\input{figs/scone.tikz.tex}}
  \caption{
    A single layer of $\SCONE$, with additional intermediate representations $\bd_0^\ell\in\cC_0, \bd_1^\ell\in\cC_1, \bd_2^\ell\in\cC_2$ included, to emphasize the structure of each $\SCONE$ layer as computing over all levels of the simplicial complex.
  }
  \label{fig:scone}
\end{figure}

We consider the task of \emph{trajectory prediction}~\citep{Benson:2016b,Wu:2017,Cordonnier:2019} for agents traveling over a simplicial complex.
A \emph{trajectory} over a simplicial complex $\cX$ with nodes $\cV$ is a sequence $[i_0,i_1,\ldots,i_{m-1}]$ of elements of $\cV$, such that $i_j$ is adjacent to $i_{j+1}$ for all $0\leq j< m-1$.
As pointed out by \citet{Ghosh:2018} and \citet{Schaub:2020}, such trajectories are naturally modeled through the lens of the Hodge Laplacian.
In particular, a trajectory viewed as an oriented $1$-chain (a linear combination of the edges of a simplicial complex) is often harmonic, \ie{}, conservative and curl-free.
This aligns with intuition: a natural walk in space will typically not backtrack on itself nor loop around points locally, and must exit most points it enters.

The trajectory prediction task considers as input a trajectory $[i_0,i_1,\ldots,i_{m-1}]$ and asks what node $i_m$ will be.
For simplicity, we do not consider the setting where a trajectory terminates, \ie{} $i_m=i_{m-1}$.

To this end, we present a neural network architecture called $\SCONE$ (Simplicial Complex Net) for trajectory prediction on simplicial complexes of dimension $2$.
We specify $\SCONE$ as a map from $\cC_1$ to $\cN(i_{m-1})$ in \cref{alg:scone}, inspired by the structure of graph convolutional networks~\citep{Bruna:2014,Defferrard:2016,Kipf:2017}.

\begin{algorithm}[tb]
  \caption{$\SCONE$ for Trajectory Prediction}
  \label{alg:scone}
  \begin{algorithmic}[1]
    \STATE \textbf{Input:} partial trajectory $[i_0,i_1,\ldots,i_{m-1}]$
    \STATE \textbf{Parameters:} \\
    boundary operators $\{\partial_k\}_{k=0}^2$ \\
    number of layers $L$ \\
    hidden dimensions $\{F_\ell\}_{\ell=0}^{L+1}, F_0=F_{L+1}=1$ \\
    weight matrices $\{\{\bW_k^\ell\in\Rnp{F_\ell}{F_{\ell+1}}\}_{\ell=0}^{L}\}_{k=0}^2$ \\
    activation function $\phi$
    \STATE \textbf{Initialize:} $\bc_1^0\in\cC_1, \bc_1^0=0$.
    \FOR{$j=0$ \textbf{to} $m-2$}
    \STATE $\bc_1^0\gets \bc_1^0+[i_j,i_{j+1}]$
    \ENDFOR
    \FOR{$\ell=0$ \textbf{to} $L-1$}
    \STATE
    \vspace{2mm}
        $\bc_1^{\ell+1} \gets \phi(\partial_2\partial_2^\top\bc_1^\ell\bW_2^\ell 
        + \bc_1^\ell\bW_1^\ell
        + \partial_1^\top\partial_1\bc_1^\ell\bW_0^\ell)$
    \vspace{2mm}
    \ENDFOR
    \STATE $\bc_0^{L+1}\gets \partial_1\bc_1^L\bW_0^L$
    \STATE $\bz\gets\mathrm{softmax}(\{[\bc_0^{L+1}]_j:j\in\cN(i_{m-1})\})$
    \STATE \textbf{Return:} $\widehat{i}_m\gets\argmax_j z_j$
  \end{algorithmic}
\end{algorithm}

\subsection{Representation of Trajectories as 1-Chains}

In order to leverage the properties of boundary operators and Hodge Laplacians of simplicial complexes, the input to $\SCONE$ needs to be a $1$-chain.
In particular, we lift the sequence of nodes $[i_0,i_1,\ldots,i_{m-1}]$ to a sequence of oriented edges $[[i_0,i_1],[i_1,i_2],\ldots,[i_{m-2},i_{m-1}]]$.
Then, we ``collapse'' the sequential structure by summing each edge in the sequence, thus yielding a $1$-chain, since each oriented edge is itself a $1$-chain in the vector space $\cC_1$.
Due to trajectories consisting of sequences of \emph{adjacent} nodes, the sequential information is mostly captured by this representation in $\cC_1$.
%

\subsection{An Admissible Architecture for 1-Chains}

Given the representation of a trajectory as a $1$-chain, we now aim to predict the next step in the trajectory prediction task.
This consists of a map from $\cC_1\to\cC_1$, followed by a mapping to $\cC_0$, and then a decision step dependent on the neighborhood of the node $i_{m-1}$.

We begin by decomposing each layer of $\SCONE$ into two steps.
First, we compute $\bc_1^{\ell+1}$ from $\bc_1^{\ell}$ as
\begin{equation}
  \bc_1^{\ell+1} \gets \phi(\partial_2\partial_2^\top\bc_1^\ell\bW_2^\ell + \bc_1^\ell\bW_1^\ell + \partial_1^\top\partial_1\bc_1^\ell\bW_0^\ell),
\end{equation}
where $\phi$ is an activation function, typically applied ``elementwise'' in the chosen oriented basis for $\cC_1$: we visualize this computation in \cref{fig:scone}.
We have abused notation here to allow each intermediate representation $\bc_1^\ell$ to consist of multiple $1$-chains, which are mixed via linear operations from the right via the matrices $\bW_k^\ell$.
After $L$ such layers, we apply the boundary map $\partial_1$, yielding a $0$-chain $\bc_0^{L+1}=\partial_1\bc_1^L\bW_0^L$.
Then, a distribution over the candidate nodes is computed via the softmax operator applied to the restriction of $\bc_0^{L+1}$ to the nodes in the neighborhood of the terminal node $\cN(i_{m-1})$.
Using sparse matrix-vector multiplication routines, the $\ell^\mathrm{th}$ layer of $\SCONE$ can be evaluated using $\bigO(|\cX_1|F_\ell F_{\ell+1}+|\cX_2|\min\{F_\ell,F_{\ell+1}\})$ operations, so that the entire architecture has a runtime of $\bigO(\sum_{\ell=0}^{L-1}(|\cX_1|F_\ell F_{\ell+1} + |\cX_2|\min\{F_\ell,F_{\ell+1}\}))$.%
  \footnote{Depending on the density of edges and triangles in $\cX$, this can be improved in practice. We leave details to the supplementary materials.}
  Moreover, the architecture of $\SCONE$ is localized, in the sense that it computes information based only on an $L$-hop (simplicial) neighborhood of the terminal node, making this architecture able to work on large simplicial complexes by only operating on a localized region of interest.

  Having defined $\SCONE$, we establish conditions under which the portion of this architecture that maps $\cC_1\to\cC_0$ is admissible.%
\footnote{
Since the final output is not a $0$-chain, but an element of $\cX_0$, admissibility of this map implies permutation and orientation \emph{invariance} for the entire architecture, rather than equivariance.
}

\begin{prop}\label{prop:admissibility}
  Assume that the activation function $\phi$ is continuous and applied elementwise.
  $\SCONE$ (\cref{alg:scone}) is admissible only if $\phi$ is an odd, nonlinear function.
\end{prop}

\begin{proof}[Proof (Sketch)]
  The proof of permutation equivariance for elementwise activation functions is directly analogous to the proof for graph neural networks, so we leave the details to the supplementary materials.

For continuous elementwise activation functions, we now consider the conditions for $\SCONE$ to satisfy orientation equivariance.
Since changes in orientation for a basis of $\cC_1$ can be expressed as a \emph{sequence} of orientation changes for individual edges, it is sufficient to study such ``single-edge'' transformations.
For some $1\leq j\leq m$, let $e_j'=-e_j$ be the reversal of the oriented edge $e_j$.
Orientation equivariance for elementwise activation functions can then be written as
\begin{equation}\label{eq:activate-orientate}
  \phi(\langle\bc_1,e_j\rangle) = \phi(-\langle\bc_1,e_j'\rangle) = -\phi(\langle\bc_1,e_j'\rangle).
\end{equation}
This condition holds for all inputs if and only if $\phi$ is an odd function.
Under these conditions, $\SCONE$ is the composition of orientation equivariant functions, and is thus orientation equivariant itself.

Finally, we consider simplicial awareness of order $2$.
Suppose $\phi$ is an odd, linear function: it is sufficient to assume that $\phi$ is the identity map.
Let a $1$-chain $\bc_1^0\in\cC_1$ be given arbitrarily.
By \cref{thm:hodge}, there exists $\bw\in\cC_0, \bx\in\kernel(\Delta_1)\subseteq\cC_1, \by\in\cC_2$ such that $\bc_1^0=\partial_1^\top\bw+\bx+\partial_2\by$.
Some simple algebra, coupled with \cref{lem:fun-thm-alg-top}, shows that when $\phi$ is the identity map, the $0$-chain at the output of $\SCONE$ is given by
\begin{equation}
    \bc_0^{L+1}=\partial_1\left(\mathrm{Id}+\partial_1\partial_1^\top\right)^L\bw.
\end{equation}
That is, the output of $\SCONE$ does not depend on $\partial_2$, and thus fails to fulfill simplicial awareness of order $2$.
However, if $\phi$ is nonlinear, \cref{lem:fun-thm-alg-top} does not come in to effect, since $\partial_1\circ\phi\circ\partial_2\neq 0$, allowing for simplicial awareness.
\end{proof}

A detailed proof can be found in the supplementary materials.
\Cref{prop:admissibility} reveals the required properties of $\phi$ for $\SCONE$ to be admissible.
In particular, we propose the use of the hyperbolic tangent activation function $\tanh$, applied to each coefficient (in the standard basis of oriented $1$-simplices) of the intermediate chains $\bc_1^\ell$; see~\cref{fig:scone}.
The fact that nonlinearities are necessary to incorporate higher-order information is in line with results in~\citet{Neuhauser:2020a,Neuhauser:2020b}, where it is shown that understanding consensus dynamics on higher-order networks must consider nonlinear behavior, lest the system be equivalently modeled as a linear process on a rescaled pairwise network.
Existing works in the development of simplicial neural networks have not discussed the necessity of having odd and nonlinear activation functions in convolutional architectures.
In particular, \citet{Ebli:2020,Bunch:2020} propose similar convolutional architectures and use $\ReLU$ activation functions.
We discuss the conditions under which these architectures can be made admissible in the supplementary material.

\section{Experiments}
\label{sec:experiments}

\begin{figure}[tb!]
  \centering
  \subfigure[][]{
    \centering
    \label{fig:exp-graphs-a}
    \resizebox{0.4\linewidth}{!}{\input{figs/synthetic-image-data/synthetic.tikz.tex}}
  }
  \subfigure[][]{%
    \centering
    \label{fig:exp-graphs-b}%
    \resizebox{0.4\linewidth}{!}{\input{figs/drifters-image-data/drifters.tikz.tex}}
  }
  
  \subfigure[][]{%
    \centering
    \label{fig:exp-graphs-c}%
    \includegraphics[width=0.4\linewidth]{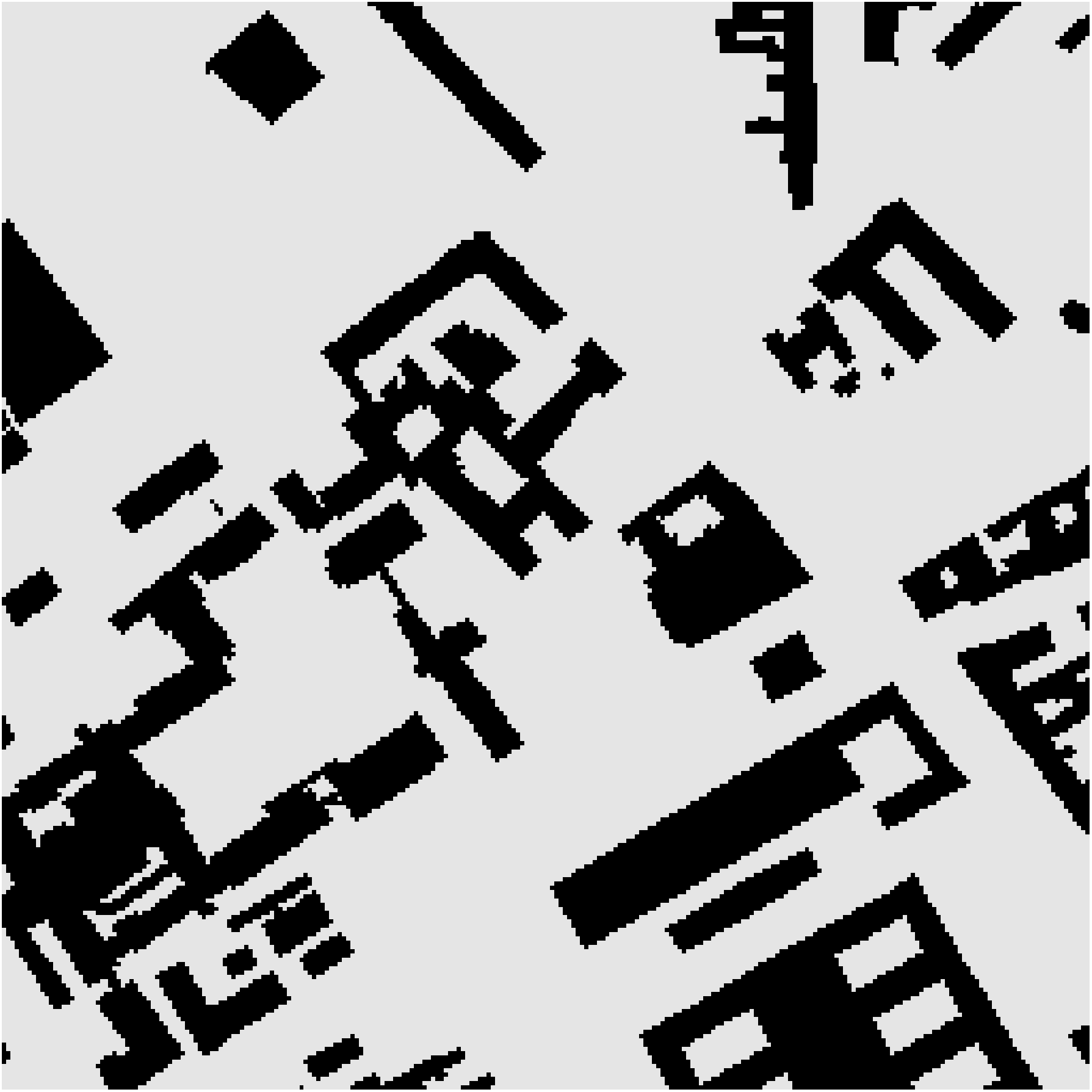}
  }
  \caption{
    Complexes used for evaluating $\SCONE$, with sample trajectories from each dataset.
    \subref{fig:exp-graphs-a}
    Synthetic example, where 400 random points in the unit square have been triangulated, followed by the removal of points from two regions.
    This yields a simplicial complex where $\dim\kernel(\Delta_1)=2$.
    \subref{fig:exp-graphs-b}
    Ocean drifters example, where the large hole (Madagascar) yields $\dim\kernel(\Delta_1)=1$.
    \subref{fig:exp-graphs-c}
    Berlin map example, where each darkly shaded region is an obstacle, yielding a hole in the underlying cubical complex.
    Figure from~\citep{Sturtevant:2012}.
  }
  \label{fig:exp-graphs}
\end{figure}

\begin{table*}[tb!]
  \caption{
    Test accuracies for trajectory prediction task.
    \subref{tab:exp-synth}
    Synthetic dataset with randomly oriented edges, comparing Markov chain, RNN, harmonic projection methods, $\SCONE$, $\mathrm{SCNN}$~\citep{Ebli:2020}, and $\mathrm{S2CCNN}$~\citep{Bunch:2020}.
    \subref{tab:exp-synth-orient}
    Synthetic dataset with manually oriented edges, compared across different activation functions for $\SCONE$.
    \subref{tab:exp-drift}
    Test accuracies for Ocean Drifters and Berlin trajectory datasets.
  }
  \label{tab:exp}
  \vskip 0.15in
  \centering
  \resizebox{\linewidth}{!}{
  \subtable[]{
    \centering
    \label{tab:exp-synth}
    \begin{small}
      \begin{sc}
        \begin{tabular}{lccccccccccc}
          \toprule
          & Markov & RNN & $\ker(\Delta_1)$ & $\ker(\partial_1)$ & $\SCONE$ & $\SCONE$ & $\SCONE$ & $\SCONE$ & $\SCONE$ & $\mathrm{SCNN}$ & $\mathrm{S2CCNN}$ \\
          &  &  & & & $\tanh$ & $\tanh$, no tri. & $\ReLU$ & $\mathrm{sigm.}$ & $\mathrm{Id}$ & & \\
          \midrule
          Std.   & 0.70 & \textbf{0.73} & 0.55 & 0.32 & 0.69 & 0.55 & 0.64 & 0.59 & 0.31 & 0.64 & 0.62 \\
          Rev.   & 0.24 & 0.01 & 0.58 & 0.21 & \textbf{0.59} & 0.49 & 0.57 & 0.57 & 0.33 & 0.48 & 0.47 \\
          Tra. & -- & -- & 0.58 & 0.40 & \textbf{0.61} & 0.58 & 0.56 & 0.53 & 0.44 & 0.42 & 0.57 \\
          \bottomrule
        \end{tabular}
      \end{sc}
    \end{small}
  }
  }
  \subtable[]{
    \centering
    \label{tab:exp-synth-orient}
    \begin{small}
      \begin{sc}
        \begin{tabular}{lcccc}
          \toprule
          & $\tanh$ & $\ReLU$ & $\mathrm{sigm.}$ & $\mathrm{Id.}$ \\
          \midrule
          Std. & 0.65 & 0.65 & \textbf{0.66} & 0.27 \\
          Rev. & \textbf{0.63} & 0.24 & 0.10 & 0.31 \\
          \bottomrule
        \end{tabular}
      \end{sc}
    \end{small}
  }
  \subtable[]{
    \centering
    \label{tab:exp-drift}
    \begin{small}
      \begin{sc}
        \begin{tabular}{lcccccc}
          \toprule
          & Markov & RNN & $\ker(\Delta_1)$ & $\SCONE$ & $\mathrm{SCNN}$ & $\mathrm{S2CCNN}$ \\
          \midrule
          Ocean & 0.45 & 0.44 & 0.45 & \textbf{0.50} & 0.18 & 0.38 \\
          Berlin & 0.76 & 0.79 & 0.50 & \textbf{0.92} & 0.85 & 0.88 \\
          \bottomrule
        \end{tabular}
      \end{sc}
    \end{small}
  }
  \vskip -0.1in
\end{table*}

\subsection{Methods}

In evaluating our proposed architecture for trajectory prediction, we consider $\SCONE$ with $3$ layers, where each layer has $F_\ell=16$ hidden features.
By default, we use the $\tanh$ activation function, but we also use $\ReLU$ and sigmoid activations to compare.
In training $\SCONE$, we minimize the cross-entropy between the softmax output $\bz$ and the ground truth final nodes in each batch of training samples.%
\footnote{Specific hyperparameters and implementation details can be found in the supplementary material.
Code available at \url{https://github.com/nglaze00/SCoNe_GCN}.}%

Seeing how $\SCONE$ consists of a map $\cC_1\to\cC_1$ followed by an application of the boundary operator and a softmax function for node selection, we compare $\SCONE$ to methods that employ other natural maps $\cC_1\to\cC_1$, followed by the same selection procedure for picking a successor node.
In particular, we consider the map that projects the input chain onto the kernel of the Hodge Laplacian $\kernel(\Delta_1)$, as well as one that projects the input chain onto the kernel of the boundary map $\kernel(\partial_1)$.
The first approach reflects the hypothesis that the harmonic subspace $\kernel(\Delta_1)$~\cf{\cref{thm:hodge}} is a natural representation for trajectories~\citep{Ghosh:2018,Schaub:2020}, and the second approach does the same while ignoring the triangular structure of the simplicial complex.
We also compare to previously proposed neural networks for simplicial data: $\mathrm{SCNN}$~\citep{Ebli:2020}, and $\mathrm{S2CCNN}$~\citep{Bunch:2020}, using leaky~$\ReLU$ or $\ReLU$ nonlinearities as done by the respective authors.

As a baseline, we consider two methods rooted in learning specific sequences, rather than a general rule parameterized by operators underpinning the supporting domain.
First, we evaluate a simple Markov chain approach where, at each node, we choose its successor based on the empirically most likely successor in the training set.
Second, we apply the RNN model of~\citet{Wu:2017}, with the adjustment of not including geometric coordinates, since we consider trajectories over abstract simplicial complexes, not assuming any geometric structure underlying it.
We emphasize that these methods are both incapable of generalizing to unseen structures, since they depend on learning decision rules for each node based on training.

\subsection{Synthetic Dataset}

\parheading{Dataset.}
Following the example of \citet{Schaub:2020}, we generate a simplicial complex by drawing 400 points uniformly at random in the unit square, and then applying a Delaunay triangulation to obtain a mesh, after which we remove all nodes and edges in two regions, pictured in~\cref{fig:exp-graphs-a}.
Then, to generate a set of trajectories, we choose a random starting point in the lower-left corner, connect it via shortest path to a random point in the upper-left, center, or lower-right region, which we then connect via shortest path to a random point in the upper-right corner.
Some examples of such trajectories are shown in~\cref{fig:exp-graphs-a}.
We generate $1000$ such trajectories for our experiment, using $800$ of them for training and $200$ for testing.

\parheading{Performance.}
To start, we evaluate the performance of all methods on the standard train/test split.
As shown in \cref{tab:exp-synth}, the RNN model performs the best, since the volume of training examples allows for the model to easily learn commonly taken paths, thus leading to good performance on the test set.
Compared to this, the kernel projection methods do worse on the test set, but do far better than random guessing.
This is due to the natural interpretation of trajectories as being characterized by the harmonic subspace of the Hodge Laplacian~\citep{Ghosh:2018,Schaub:2020}.
Moreover, projecting onto the kernel of the Hodge Laplacian $\kernel(\Delta_1)$ significantly outperforms only projecting onto the kernel of the boundary map $\kernel(\partial_1)$.
This indicates that incorporating information from $\cX_2$ (triangles) is important for this problem, since it allows the harmonic subspace to capture more interesting homological structure.

Finally, we evaluate $\SCONE$ using different nonlinear activation functions and incorporation of simplicial information.
When using $\tanh$ activation functions, it is clear that including the triangles in the model improves performance, in line with what we observed for the kernel projection methods.
Moreover, we see that using the $\tanh$ activation function over the sigmoid or $\ReLU$ activation functions also improves performance, presumably due to the fact that $\tanh$ is odd, as required for admissibility in~\cref{prop:admissibility}.
Similarly, although the identity ($\mathrm{Id.}$) activation function is odd, it is linear, and thus does not satisfy~\cref{prop:admissibility}. This leads to poor performance in all experiments.
We also observe that $\SCONE$ tends to outperform the $\mathrm{SCNN}$ and $\mathrm{S2CCNN}$ models, perhaps due to the extra regularization imposed by admissibility.

\parheading{Testing Generalization.}
We demonstrate the generalization properties of $\SCONE$ in two ways, with the common feature of manipulating the training and test sets in order for them to have a mismatch in their characteristics.

First, we evaluate these methods on a ``reversed'' test set.
That is, we keep the training set the same as before, but reverse the direction of the trajectories in the test set.
Therefore, a method that is overly dependent on ``memorizing'' a particular direction for the trajectories will be expected to fare poorly, while methods that leverage more fundamental features relating to the homology of the simplicial complex are expected to perform better.
Based on the results in \cref{tab:exp-synth}, we see two things: including triangles in the architecture improves performance, and admissible methods outperform inadmissible methods.
In particular, $\SCONE$ using $\tanh$ activation and incorporating triangles performs similarly to the projection onto $\kernel(\Delta_1)$, and both of these methods outperform competing approaches.

Second, we evaluate how well $\SCONE$ generalizes to trajectories over unseen simplicial structures.
To do this, we restrict the training set to trajectories running along the upper-left region of $\cX$, and similarly restrict the testing set to trajectories spanning the lower-right region of $\cX$.
Although in this case $\SCONE$ is still being applied to the same simplicial complex that it was trained on, the locality of the architecture means that the testing set is essentially an unseen structure.
We see again in \cref{tab:exp-synth} that $\SCONE$ with $\tanh$ activation outperforms sigmoid and $\ReLU$ activations, illustrating the utility of admissibility for designing architectures that generalize well.
Note that the Markov chain and the RNN cannot be tested on unseen data.

Finally, to demonstrate the sensitivity of architectures that are not orientation equivariant, we repeat these experiments on the same dataset, except with edge orientations selected carefully in order to reflect the general direction of the training set.
That is, using the geometric position of the nodes in \cref{fig:exp-graphs-a}, we label each node based on the sum of its $x$ and $y$ coordinates, and then orient each edge to be increasing with respect to the ordering of the nodes.
This yields a set of oriented edges that ``point'' from the lower-left of the simplicial complex to the upper-right.
Since the training set consists of trajectories that also follow this direction, the coefficients in the representation of each trajectory as a vector in $\cC_1$ will be overwhelmingly non-negative.

By manipulating the edge orientation in this way, we have artificially introduced a rule that would work well for the training set, but violates orientation equivariance.
Since architectures that violate orientation equivariance are capable of learning such rules, we see in \cref{tab:exp-synth-orient} that the $\ReLU$ and sigmoidal architectures perform similarly to the admissible $\tanh$ architecture when the test set matches this rule, with the sigmoidal architecture performing marginally better than the others.
However, as soon as the data does not follow this rule, as in the ``reversed'' case, the non-admissible architectures markedly fail, with the sigmoidal and $\ReLU$ activations not satisfying orientation equivariance, and the identity activation not satisfying simplicial awareness.

\subsection{Real Data}

\parheading{Datasets.}
We consider the trajectory prediction problem for the Global Drifter Program dataset, localized around Madagascar.\footnote{Data available from NOAA/AOML at \url{http://www.aoml.noaa.gov/envids/gld/} and as supplementary material.}
This dataset consists of buoys whose coordinates are logged every 12 hours.
We tile the ocean around Madagascar with hexagons and consider the trajectory of buoys based on their presence in hexagonal tiles at each logged moment in time, following the methodology of~\citet{Schaub:2020}.
Treating each tile as a node, drawing an edge between adjacent tiles, and filling in all such planar triangles yields a natural simplicial structure, as pictured in~\cref{fig:exp-graphs-b}.
Indeed, the homology of the complex shows a large hole, corresponding to the island of Madagascar.

Additionally, we consider a map of a section of Berlin~\citep{Sturtevant:2012}, where each point on a grid is designated impassible or passible based on the presence of an obstacle, as pictured in~\cref{fig:exp-graphs-c}.
For a set of trajectories, we consider a set of $1000$ shortest paths between random pairs of points in the largest connected components, divided into an $80/20$ train/test split.
Since the geometry of this map is given as a grid, it is naturally modeled as a \emph{cubical complex,} rather than a simplicial complex.
The boundary operators in this setting are quite similar to the simplicial case, highlighting the flexibility of our proposed architecture for more general chain complexes.
We discuss the details of this in the supplemental material.

\parheading{Results.}
Across the competing methods, $\SCONE$ yields the best trajectory prediction, as shown in \cref{tab:exp-drift}.
Indeed, by using admissibility as a guiding principle in designing architectures that respect the symmetries of the underlying chain complex, as well as ensuring complete integration of the simplicial structure, we achieve greater prediction accuracy than other methods.
Interestingly, we again outperform the method of projecting onto $\kernel(\Delta_1)$, which is itself an admissible approach, suggesting that there is utility in the regularization imposed by the successive local aggregations of $\SCONE$, and in appropriate weighting of the different components of the Hodge Laplacian.

\section{Conclusion}

A core component of graph neural networks is their equivariance to permutations of the nodes of the graph on which they act.
Designing architectures that respect properties such as this enables the design of systems that transfer and scale well.
By considering additional symmetries (orientation equivariance) and truly accounting for higher-order structures (simplicial awareness), we construct principled, generalizable neural networks for data supported on simplicial complexes.

\section*{Acknowledgements}

This work was partially supported by NSF under award \href{https://www.nsf.gov/awardsearch/showAward?AWD_ID=2008555&HistoricalAwards=false}{CCF-2008555}.
Research was sponsored by the Army Research Office and was accomplished under Cooperative Agreement Number W911NF-19-2-0269.
The views and conclusions contained in this document are those of the authors and should not be interpreted as representing the official policies, either expressed or implied, of the Army Research Office or the U.S. Government.
The U.S. Government is authorized to reproduce and distribute reprints for Government purposes notwithstanding any copyright notation herein.
TMR was partially supported by the Ken Kennedy Institute 2020/21 Exxon-Mobil Graduate Fellowship.
We would like to thank Michael Schaub for helpfully providing a cleaned dataset for the Ocean Drifters experiment.

\newpage

\onecolumn

\appendix


\renewcommand{\theequation}{S-\arabic{equation}}

\renewcommand{\theprop}{S-\arabic{prop}}
\renewcommand{\thetheorem}{S-\arabic{theorem}}
\renewcommand{\thelemma}{S-\arabic{lemma}}
\renewcommand{\thecoro}{S-\arabic{coro}}
\renewcommand{\theproblem}{S-\arabic{problem}}
\renewcommand{\theassump}{S-\arabic{assump}}
\renewcommand{\theremark}{S-\arabic{remark}}
\renewcommand{\thedefn}{S-\arabic{defn}}

\renewcommand{\thefigure}{S-\arabic{figure}}
\renewcommand{\thetable}{S-\arabic{table}}

\renewcommand{\thealgorithm}{S-\arabic{algorithm}}

\section*{Supplementary Material}

In this supplement, we discuss practical concerns for the implementation of simplicial neural networks.
In \cref{app:represent}, we discuss how to represent chains as real vectors and boundary maps as matrices, as well as how to implement activation functions using this representation.
Then, using these convenient representations, in \cref{app:implement} we redefine $\SCONE$ using real vectors and matrices, rather than vectors and linear maps in the chain complex $\cC$, and then specify the hyperparameters used in our experiments.
We briefly discuss the implementation of the nullspace projection methods in \cref{app:nullspace-projection}, followed by a computational complexity analysis of $\SCONE$ in \cref{app:complexity}.
The necessity of odd, nonlinear activation functions as stated in \cref{prop:admissibility} is proven in \cref{app:proof-admissibility}.
The use of odd activation functions is contrasted with existing work in \cref{app:past-admissibility}.
We provide more details on \cref{thm:hodge} as it pertains to $\cC_1$ in \cref{app:hodge}, before finally discussing an implementation of $\SCONE$ for cubical complexes in \cref{app:cubism}.

\section{Representing Chains and Boundary Maps as Vectors and Matrices}\label{app:represent}

In \cref{alg:scone}, we specify $\SCONE$ in terms of boundary maps acting on chains interleaved with matrix multiplication from the right and activation functions.
Here, we describe  simple procedures for constructing vector representations of $k-$chains, as well as matrix representations of boundary maps that act on said representations.

Let $\cX$ be a simplicial complex over a set of nodes $\cX_0$, with edges $\cX_1$ and triangles $\cX_2$.
Begin by labeling the vertices $\cX_0$ with the integers $\{1,\ldots,n\}$.
Letting $m=|\cX_1|$, sort the edges lexicographically by their constituent nodes and label them accordingly with the integers $\{1,\ldots,m\}$.
Similarly, letting $p=|\cX_2|$, sort the triangles lexicographically by their constituent nodes and label them with the integers $\{1,\ldots,p\}$.
For each edge $e=\{i,j\}$, for $i,j\in\cX_0$, assign to $e$ the orientation $[i,j], i<j$.
Similarly, for each triangle $t=\{i,j,k\}$, assign to $t$ the orientation $[i,j,k], i<j<k$.

\subsection{Chains as Real Vectors}

We represent a $1-$chain as a vector in $\Rn{m}$.
Let $\cE=\{e_1,e_2,\ldots,e_m\}$ be the set of labeled, oriented edges, and suppose for real coefficients $\alpha_k$ we have a $1-$chain $\bc_1=\sum_{i=1}^m \alpha_i e_i$.
We identify $\bc_1$ with a vector in $\Rn{m}$, so that in this representation $[\bc_1]_i=\alpha_i$ for $1\leq i\leq m$.
Similar representations as vectors of real numbers can be derived for other $k-$chains.

\subsection{Boundary Maps as Matrices}

With chains admitting natural representations as real vectors, we represent the boundary operators as matrices.
First, recall that $\partial_1:\cC_1\to\cC_0$, and $\cC_1,\cC_0$ are $m,n-$dimensional vector spaces, respectively.
A matrix representation of $\partial_1$, denoted by the matrix $\bB_1$, then, must match these dimensions, so that $\bB_1\in\Rnp{n}{m}$.
The entries of $\bB_1$ are defined as follows.
Again, let $\cE=\{e_1,e_2,\ldots,e_m\}$ be the set of labeled, oriented edges, and let $\cV=\{v_1,v_2,\ldots,v_n\}$ be the set of labeled nodes.
Then, the entries of $\bB_1$ are given by
\begin{equation}
  [\bB_1]_{ij}=
  \begin{cases}
    -1 & e_j=[v_i,\cdot] \\
    1  & e_j=[\cdot,v_i] \\
    0  & \text{otherwise}.
  \end{cases}
\end{equation}
Observe that this aligns with the definition of the boundary map in \cref{eq:boundary}.
Under this definition, $\bB_1$ is precisely the \emph{signed incidence matrix} of the graph $(\cX_0,\cX_1)$.
Moreover, the adjoint of $\partial_1$ is represented as the transpose of this matrix, \ie{} $\bB_1^\top$.

The definition of the matrix representation of $\partial_2$ is slightly more complex.
Let $\{t_1,t_2,\ldots,t_p\}$ be the set of labeled, oriented triangles.
Since $\partial_2:\cC_2\to\cC_1$ is a map from a $p-$dimensional vector space to an $m-$dimensional vector space, the matrix representation $\bB_2$ must also match this, so that $\bB_2\in\Rnp{m}{p}$.
For each $1\leq j\leq m$ and $1\leq k\leq p$, let $e_j=[i_0,i_1]$ be the $j^\mathrm{th}$ oriented edge.
Then, the entries of $\bB_2$ are defined as
\begin{equation}
  [\bB_2]_{jk}=
  \begin{cases}
    -1 & t_k=[i_0,\cdot,i_2] \\
    1  & t_k=[\cdot,i_0,i_1] \\
    1  & t_k=[i_0,i_1,\cdot] \\
    0  & \text{otherwise}.
  \end{cases}
\end{equation}
Again, one can check that this coincides with the definition of the boundary map in \cref{eq:boundary}, and that $\bB_1\bB_2=0$ in accordance with \cref{lem:fun-thm-alg-top}.
Observe that the first Hodge Laplacian $\Delta_1$ can be written as the matrix $\bL_1=\bB_1^\top\bB_1+\bB_2\bB_2^\top$.

\subsection{Application of Elementwise Activation Functions}

We construct $\SCONE$ using an odd, elementwise activation function $\phi$ applied to vectors in $\cC_1$.
Based on the representation of chains as vectors of real numbers defined before, the application of the activation function is fairly obvious: simply apply $\phi$ to each real-valued coordinate in the vector.
For the sake of completeness, we outline here how activation functions can be applied to $1-$chains without using the intermediate representation as real-vectors.

Let $\{e_1,e_2,\ldots,e_m\}$ denote the set of $m=|\cX_1|$ oriented edges of a simplicial complex, so that any $1-$chain $\bc_1\in\cC_1$ is a unique linear combination of these oriented edges.
We extend the activation function $\phi:\R\to\R$ to act on $\cC_1$ as follows:
\begin{equation}\label{eq:elementwise}
    \phi(\bc_1)=\sum_{\ell=1}^m\phi(\langle\bc_1,e_\ell\rangle)\cdot e_\ell.
\end{equation}
Observe that this can be equivalently computed by representing $\bc_1$ as a real vector and applying $\phi$ elementwise, as desired.

\section{Implementation of SCoNe}\label{app:implement}

\begin{algorithm}[tb]
  \caption{$\SCONE$ for Trajectory Prediction Defined in Terms of Real Vectors and Matrices}
  \label{alg:scone-real}
  \begin{algorithmic}[1]
    \STATE \textbf{Input:} partial trajectory $[i_0,i_1,\ldots,i_{m-1}]$
    \STATE \textbf{Parameters:} \\
    boundary matrices $\{\bB_k\}_{k=0}^2$ for oriented edges $\cE=\{e_1,e_2,\ldots,e_m\}$ \\
    number of layers $L$ \\
    hidden dimensions $\{F_\ell\}_{\ell=0}^{L+1}, F_0=F_{L+1}=1$ \\
    weight matrices $\{\{\bW_k^\ell\in\Rnp{F_\ell}{F_{\ell+1}}\}_{\ell=0}^{L}\}_{k=0}^2$ \\
    activation function $\phi$
    \STATE \textbf{Initialize:} $\bc_1^0\in\Rn{|\cE|}, \bc_1^0=0$.
    \FOR{$j=0$ \textbf{to} $m-2$}
    \IF{$[i_j,i_{j+1}]\in\cE$}
    \STATE Choose $\ell$ such that $e_\ell=[i_j,i_{j+1}]$
    \STATE $[\bc_1^0]_\ell\gets 1$
    \ELSE
    \STATE Choose $\ell$ such that $e_\ell=[i_{j+1},i_j]$
    \STATE $[\bc_1^0]_\ell\gets -1$
    \ENDIF
    \ENDFOR
    \FOR{$\ell=0$ \textbf{to} $L-1$}
    \STATE
    \begin{equation}\label{eq:scone-layer}
      \begin{aligned}
        \bc_1^{\ell+1} &\gets \phi(\bB_2\bB_2^\top\bc_1^\ell\bW_2^\ell \\
        &\qquad+ \bc_1^\ell\bW_1^\ell \\
        &\qquad+ \bB_1^\top\bB_1\bc_1^\ell\bW_0^\ell)
      \end{aligned}
    \end{equation}
    \ENDFOR
    \STATE $\bc_0^{L+1}\gets \bB_1\bc_1^L\bW_0^L$
    \STATE $\bz\gets\mathrm{softmax}(\{[\bc_0^{L+1}]_j:j\in\cN(i_{m-1})\})$
    \STATE \textbf{Return:} $\widehat{i}_m\gets\argmax_j z_j$
  \end{algorithmic}
\end{algorithm}

With the representation of $1-$chains as real vectors and boundary maps as matrices defined, we redefine \cref{alg:scone} in terms of these real vectors and matrices, detailed in \cref{alg:scone-real}.
For each experiment, we trained $\SCONE$ using $L=3$ layers, with hidden dimensions $F_\ell=16$.
It was trained using the Adam optimizer~\citep{Kingma:2015}, using the default parameters of $\beta_1=0.9, \beta_2=0.99$, for the cross-entropy loss between the output of $\SCONE$ and the true labeled successor node, with an additional weight decay term with coefficient~\num{5e-5}~\citep{Krogh:1992}.
The remaining hyperparameters varied across experiments, and are listed in~\cref{tab:to-rule-them-all}.

\begin{table*}[tb]
\caption{
Hyperparameters for synthetic and real data experiments.
The hyperparameters for the synthetic experiments were identical between the standard and reversed settings, and changed slightly in the transfer learning setting.
Hyperparameters were selected by hand, which was reasonable due to the simple nature of the proposed architecture.
}
\label{tab:to-rule-them-all}
\vskip 0.15in
\begin{center}
\begin{small}
\begin{sc}
\begin{tabular}{lrrr}
  \toprule
  Parameter & Synthetic (Std./Rev.) & Synthetic (Trans.) & Drifters \\
  \midrule
  Learning rate & 0.001 & 0.001 & 0.0025 \\
  Training samples & 800 & 333 & 160 \\
  Test samples & 200 & 333 & 40 \\
  Epochs & 500 & 1000 & 425 \\
  \bottomrule
\end{tabular}
\end{sc}
\end{small}
\end{center}
\vskip -0.1in
\end{table*}

\section{Implementation of Nullspace Projection Methods}
\label{app:nullspace-projection}

Examining the implementation of $\SCONE$, we see that the architecture consists of a map from $1$-chains to $1$-chains, followed by the application of the boundary map $\partial_1$ to obtain a $0$-chain.
In \cref{sec:experiments}, we compared $\SCONE$ to methods based on projecting $1$-chains into the kernel of $\partial_1$ or $\Delta_1$, following the work of \citet{Schaub:2020}.
Of course, by \cref{lem:fun-thm-alg-top}, applying $\partial_1$ to this would always yield the zero vector, which is useless for prediction tasks.
Therefore, we adjusted this by constructing a version of the boundary map restricted to the edges adjacent to the terminal node.
In practice, this amounts to choosing the edge with the largest outgoing flow from $i_{m-1}$, then predicting the next node as the endpoint of that edge.

\section{Computational Complexity Analysis}
\label{app:complexity}

We establish the $\bigO(|\cX_1|F_{\ell}F_{\ell+1}+|\cX_2|\min\{F_\ell,F_{\ell+1})$ runtime of the $\ell^\mathrm{th}$ layer of $\SCONE$ here.
Observe in \cref{alg:scone-real} that the $\ell^\mathrm{th}$ layer of $\SCONE$ maps a matrix in $\Rnp{|\cX_1|}{F_\ell}$ to a matrix in $\Rnp{|\cX_1|}{F_\ell}$ by multiplying the argument by $F_\ell\times F_{\ell+1}$ matrices from the right, and combinations of boundary maps from the left, followed by an aggregation and activation step.
One such way to compute this is by considering intermediate representations $\bd^\ell$, defined as follows.
\begin{align}
  \bd_0^\ell &= \bB_1^\top\bB_1\bc_1^\ell\bW_0^\ell \\
  \bd_1^\ell &= \bc_1^\ell\bW_1^\ell \\
  \bd_2^\ell &= \bB_2\bB_2^\top\bc_1^\ell\bW_2^\ell.
\end{align}
Then, the output is computed by
\begin{equation}
  \bc_1^{\ell+1} = \phi\left(\bd_0^\ell+\bd_1^\ell+\bd_2^\ell\right).
\end{equation}
The complexity of computing $\bc_1^{\ell+1}$, then, is equal to the sum of the complexities of computing $\bd_0^\ell,\bd_1^\ell,\bd_2^\ell$ and the complexity of taking their sum and applying $\phi$.

\begin{table*}[tb]
\caption{
  Computational complexities for computing intermediate representations of a single $\SCONE$ layer.
  We indicate the order of operations in the ``Expression'' column using parentheses, which yields different complexities based on the size of each level of the simplicial complex $\cX$.
  The best conditions for each expression in terms of the relative sizes of the levels of $\cX$ are listed in the column ``Best use case.''
}
\label{tab:complexity}
\vskip 0.15in
\begin{center}
\begin{small}
\begin{sc}
\begin{tabular}{cccc}
  \toprule
  Quantity & Expression & Complexity & Best use case \\
  \midrule
  \multirow{3}{*}{$\bd_0^\ell$} & $\bB_1^\top(\bB_1(\bc_1^\ell\bW_0^\ell))$ & $\bigO(|\cX_1|F_\ell F_{\ell+1})$ & $|\cX_1|\approx|\cX_0|$\\
  & $\bB_1^\top((\bB_1\bc_1^\ell)\bW_0^\ell)$ & $\bigO(|\cX_0|F_\ell F_{\ell+1}+|\cX_1|(F_\ell+F_{\ell+1}))$ & $|\cX_1|\gg|\cX_0|$ \\
  & ($\bB_1^\top(\bB_1\bc_1^\ell))\bW_0^\ell$ & $\bigO(|\cX_1|F_\ell F_{\ell+1})$ & $|\cX_1|\approx|\cX_0|$ \\
  \midrule
  $\bd_1^\ell$ & $\bc_1^\ell\bW_1^\ell$ & $\bigO(|\cX_1|F_\ell F_{\ell+1})$ & \\
  \midrule
  \multirow{3}{*}{$\bd_2^\ell$} & $\bB_2 (\bB_2^\top(\bc_1^\ell\bW_2^\ell))$ & $\bigO(|\cX_1|F_\ell F_{\ell+1}+|\cX_2|F_{\ell+1})$ & $|\cX_2|\gg|\cX_1|$ \\
  & $\bB_2((\bB_2^\top\bc_1^\ell)\bW_2^\ell)$ & $\bigO(|\cX_2|(F_\ell F_{\ell+1}))$ & $|\cX_2|\leq|\cX_1|$ \\
  & $(\bB_2(\bB_2^\top\bc_1^\ell))\bW_2^\ell$ & $\bigO(|\cX_1|F_\ell F_{\ell+1}+|\cX_2|F_\ell)$ & $|\cX_2|\gg|\cX_1|$ \\
  \bottomrule
\end{tabular}
\end{sc}
\end{small}
\end{center}
\vskip -0.1in
\end{table*}

Observing that the matrix $\bB_1$ is typically a sparse matrix with $\bigO(|\cX_1|)$ nonzero entries, and $\bB_2$ is typically a sparse matrix with $\bigO(|\cX_2|)$ nonzero entries, complexities for computing $\bd_0^\ell,\bd_1^\ell,\bd_2^\ell$ can be derived in a straightforward manner.
We remark that the order of operations has an impact on complexity, as applying the boundary map to an input may increase or decrease the dimension, which has an impact on the application of the weight matrix $\bW_k^\ell$.
For instance, if a simplicial complex has very few triangles, one should multiply $\bc_1^\ell$ by $\bB_2^\top$ first when computing $\bd_2^\ell$, since the lower-dimensional space $\cC_2$ will have lower complexity when multiplying from the right by $\bW_2^\ell$.
We gather the complexities of all possible ways to evaluate these expressions in \cref{tab:complexity}.
Observing that each $\bd_k^\ell$ can be evaluated in $\bigO(|\cX_1|F_\ell F_{\ell+1}+|\cX_2|\min\{F_\ell,F_{\ell+1}\})$ time yields the desired complexity, since the application of $\phi$ is only $\bigO(|\cX_1|F_{\ell+1})$.

\section{Proof of \cref{prop:admissibility}}\label{app:proof-admissibility}

We provide a proof of necessary and sufficient conditions for $\SCONE$ to be an admissible architecture, as stated in \cref{prop:admissibility}.
We first establish the following auxiliary result.
\begin{prop}\label{prop:self-similar-continuity}
  Let $\phi:\R\to\R$ be a function defined on the reals.
  Suppose that for some pair of real numbers $a,b$ such that $a\neq -b$ and $a\neq b$, we have
  \begin{equation}
    \begin{aligned}
      \phi(a) &= -\phi(b) \\
      \phi(a) &\neq 0,
    \end{aligned}
  \end{equation}
  and for all $\gamma\in\R$
  \begin{equation}
    \phi(\gamma a)=-\phi(\gamma b).
  \end{equation}
  Under these conditions, $\phi$ is not continuous.
\end{prop}
\begin{proof}
  Since $a\neq -b$ and $a\neq b$, $|a|\neq|b|$.
  Without loss of generality, assume $|a|>|b|$, so that putting $\gamma=\frac{b}{a}$ satisfies $|\gamma|<1$.
  We then have
  \begin{align}
    \phi(\gamma a) &= \phi(b) \\
    \gamma^j b &= \gamma^{j+1} a
  \end{align}
  for all nonnegative integers $j$.
  An immediate consequence of this is that for all nonnegative integers $j$, we have
  \begin{equation}\label{eq:alternating-equality}
    \phi(a) = (-1)^j\phi\left(\gamma^j a\right).
  \end{equation}
  Consider the sequence $\{c_j\}_{j=1}^\infty$ where $c_j=\gamma^j b$ for each $j$.
  Observe that as $j\to\infty$, $c_j\to 0$ due to the fact that $|\gamma|<1$.
  Moreover, by~\eqref{eq:alternating-equality}, we have that $\phi(c_j)=-\phi(c_{j+1})$ for each $j$.
  Therefore, for all $j>0$,
  \begin{equation}
    \left|\phi(c_j)-\phi(c_{j+1})\right|=2\left|\phi(a)\right|>0.
  \end{equation}
  That is, although the sequence $\{c_j\}_{j=1}^\infty$ converges to $0$, the sequence $\{\phi(c_j)\}_{j=1}^\infty$ does not converge.
  Therefore, $\phi$ is not continuous at the point $0$, so that $\phi$ is not a continuous function, as desired.
\end{proof}

We now prove~\cref{prop:admissibility}~(restated below).
\begin{prop}[{Restatement of~\cref{prop:admissibility}}]
Assume that the activation function $\phi$ is continuous and applied elementwise.
If $\SCONE$ (as defined in~\cref{alg:scone}) is admissible, $\phi$ must be an odd and nonlinear function.
\end{prop}

\begin{proof}
We prove this statement by first establishing that $\SCONE$ is permutation equivariant by virtue of $\phi$ being applied elementwise, and then show that orientation equivariance and simplicial awareness are satisfied only if $\phi$ is odd and nonlinear, respectively.
  
\parheading{Permutation equivariance.}

Let $\{e_1,e_2,\ldots,e_m\}$ be a chosen orientation of edges for a simplicial complex $\cX$, where each $\{e_\ell\}_{j=1}^m$ is an oriented $1$-simplex (and thus is a $1$-chain).
The activation function $\phi$ is applied to arbitrary $\bc_1\in\cC_1$ as follows:
\begin{equation}\label{eq:activate-applicate}
  \phi(\bc_1)=\sum_{j=1}^m\phi(\langle\bc_1,e_\ell\rangle)e_\ell.
\end{equation}
Observe that $\phi$ commutes with permutation matrices, \ie{} $\phi(\bP_1\bc_1)=\bP_1\phi(\bc_1)$.
In particular, we take
\begin{equation}\label{eq:scone-activation-permute1}
  \bc_1^{\ell+1} \gets \phi(\partial_2\partial_2^\top\bc_1^\ell\bW_2^\ell        + \bc_1^\ell\bW_1^\ell + \partial_1^\top\partial_1\bc_1^\ell\bW_0^\ell).
\end{equation}
Letting $\cP=\{\bP_k\}_{k=0}^2$ as in~\cref{property:perm-equi}, we consider a version of~\eqref{eq:scone-activation-permute1} where the input $\bc_1^\ell$ and the boundary maps are permuted by $\cP$:
\begin{equation}\label{eq:scone-activation-permute2}
  \begin{gathered}
    \phi(\bP_1\partial_2\bP_2^\top\bP_2\partial_2^\top\bP_1^\top\bP_1\bc_1^\ell\bW_2^\ell + \bP_1\bc_1^\ell\bW_1^\ell + \bP_1\partial_1^\top\bP_0^\top\bP_0\partial_1\bP_1^\top\bP_1\bc_1^\ell\bW_0^\ell) = \\
    \phi(\bP_1\partial_2\partial_2^\top\bc_1^\ell\bW_2^\ell + \bP_1\bc_1^\ell\bW_1^\ell + \bP_1\partial_1^\top\partial_1\bc_1^\ell\bW_0^\ell) = \\
    \phi(\bP_1(\partial_2\partial_2^\top\bc_1^\ell\bW_2^\ell + \bc_1^\ell\bW_1^\ell + \partial_1^\top\partial_1\bc_1^\ell\bW_0^\ell)) = \\
    \bP_1\phi(\partial_2\partial_2^\top\bc_1^\ell\bW_2^\ell + \bc_1^\ell\bW_1^\ell + \partial_1^\top\partial_1\bc_1^\ell\bW_0^\ell) = \\
    \bP_1\bc_1^{\ell+1}.
  \end{gathered}
\end{equation}
That is to say, each layer of $\SCONE$ is equivariant to joint permutations of the input and the boundary maps, so that the entirety of $\SCONE$ is permutation equivariant, as desired.

\parheading{Orientation equivariance.}

To show that $\phi$ being odd is a necessary condition for orientation equivariance, suppose that $\phi:\R\to\R$ is not odd and is continuous.
Then, there exists $x\in\R$ such that $\phi(x)\neq-\phi(-x)$.
This implies that either $\phi(x)$ or $\phi(-x)$ is nonzero.
Without loss of generality, then, suppose $\phi(x)\neq 0$.

Take $\cX$ to be a simplicial complex with two nodes and a single edge connecting them: $\cX=\{\{i_0\},\{i_1\},\{i_0,i_1\}\}$.
Finding it convenient to represent $1$-chains in this case as real numbers, let $\bc_1^0=x$ be a $1$-chain, and let $\{\{\bW_k^\ell\}_{\ell=0}^L\}_{k=0}^2$ all be contained in $\Rnn{1}$, so that their application is equivalent to scalar multiplication.
Set $\bW_1^\ell=\bW_2^\ell=0$ for all $\ell$, and denote by $w^\ell$ the scalar component of $\bW_0^\ell$ for each $\ell$.
We now show inductively that for any nonnegative integer $L$, there exists an $L$-layer $\SCONE$ architecture with coefficients $\{w^\ell\}_{\ell=0}^{L-1}$ that does not satisfy orientation equivariance.

For the base case, let $L=1$, and consider a $1$-layer $\SCONE$ architecture:
\begin{equation}
  \bc_1^1 = \phi(w^0\bc_1^0).
\end{equation}
Setting $w^0=1$ yields $\bc_1^1=\phi(x)$.
One can easily see that this is not orientation equivariant, since $\phi(x)\neq-\phi(-x)$.

For the inductive step, suppose that $L>1$, and that a $\SCONE$ architecture with $L-1$ layers and coefficients $\{w^\ell\}_{\ell=0}^{L-2}$ is not orientation equivariant for the input $\bc_1^0$: denote the differently oriented outputs as $\bc_{1+}^{L-1}$ and $\bc_{1-}^{L-1}$, so that $\bc_{1+}^{L-1}\neq-\bc_{1-}^{L-1}$.
If $\bc_{1+}^{L-1}=\bc_{1-}^{L-1}$, implying that $\bc_{1+}^{L-1}\neq 0$, take $w^{L-1}=x/\bc_{1+}^{L-1}$, so that
\begin{equation}
  \phi(w^{L-1}\bc_{1+}^{L-1})=\phi(x)\neq-\phi(x)=-(w^{L-1}\bc_{1-}^{L-1}),
\end{equation}
yielding an architecture that is not orientation equivariant.
Otherwise, suppose for the sake of contradiction that $\bc_{1+}^{L-1}\neq\bc_{1-}^{L-1}$ and for all $w^{L-1}\in\R$ we have orientation equivariance for the input $\bc_1^0$.
That is,
\begin{equation}
  \phi(w^{L-1}\bc_{1+}^{L-1}) = -\phi(w^{L-1}\bc_{1-}^{L-1}).
\end{equation}
Since at least one of the $1$-chains $\bc_{1+}^{L-1},\bc_{1-}^{L-1}$ is nonzero, one can always choose $w^{L-1}$ such that $\phi(w^{L-1}\bc_{1?}^{L-1})$ is nonzero, where $\bc_{1?}^{L-1}$ denotes said nonzero $1$-chain.
By \cref{prop:self-similar-continuity}, this implies that $\phi$ is not a continuous function, yielding a contradiction, as desired.
Thus, under these conditions, there exists a $\SCONE$ architecture that is not orientation equivariant.

\parheading{Simplicial awareness.}

Finally, we consider simplicial awareness of order $2$.
Suppose $\phi$ is an odd, linear function: it is sufficient to assume that $\phi$ is the identity map.
Let a $1$-chain $\bc_1^0\in\cC_1$ be given arbitrarily.
By \cref{thm:hodge}, there exists $\bw\in\cC_0, \bx\in\kernel(\Delta_1), \by\in\cC_2$ such that $\bc_1^0=\partial_1^\top\bw+\bx+\partial_2\by$.
Some simple algebra, coupled with \cref{lem:fun-thm-alg-top}, shows that when $\phi$ is the identity map, the $0$-chain at the output of $\SCONE$ is given by
\begin{equation}
  \bc_0^{L+1}=\partial_1\sum_{j=0}^L(\partial_1^\top\partial_1)^L\partial_1^\top\bw\omega(j),
\end{equation}
where each $\omega(j)$ can be written as a polynomial of the weight matrices (always yielding a $1\times 1$ matrix, due to the constraint $F_0=F_{L+1}=1$).
That is, the output of $\SCONE$ does not depend on $\partial_2$, and thus fails to fulfill simplicial awareness of order $2$.
However, if $\phi$ is nonlinear, \cref{lem:fun-thm-alg-top} does not come in to effect, since $\partial_1\circ\phi\circ\partial_2\neq 0$, allowing for simplicial awareness, as desired.
\end{proof}

\section{Admissibility of Previous Simplicial Neural Networks}\label{app:past-admissibility}

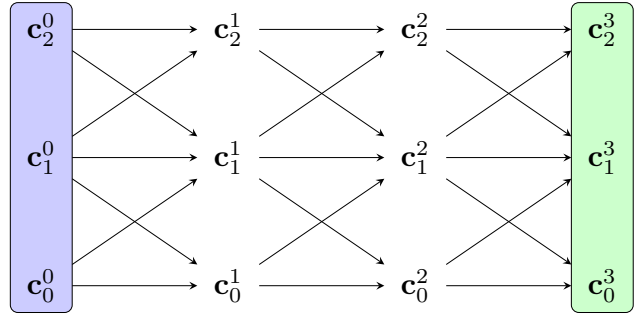
\begin{figure*}
  \centering
  \resizebox{0.5\linewidth}{!}{\input{figs/bunch.tikz.tex}}
  \caption{
    General structure of the simplicial $2$-complex convolutional neural network of \citet{Bunch:2020}.
    Each horizontal arrow corresponds to the application of a normalized Hodge Laplacian, and diagonal arrows correspond to the application of normalized boundary or coboundary maps.
    At each node, the inputs are summed then passed through the activation function $\phi$.
  }
  \label{fig:bunch}
\end{figure*}

Simplicial neural networks similar to $\SCONE$ were previously proposed by \citet{Ebli:2020,Bunch:2020}.
The convolutional layer for $k$-(co)chains of \citet{Ebli:2020} takes the following form:
\begin{equation}\label{eq:ebli-layer}
  \bc_k^{\ell+1}\gets\phi\left(\sum_{j=0}^N\Delta_k^j\bc_k^\ell\bW_j^\ell\right),
\end{equation}
where $\phi$ is an activation function applied according to the standard coordinate basis as in \eqref{eq:elementwise}, and $\{\bW_j^\ell\}$ are trainable weight matrices.
Observe that \eqref{eq:ebli-layer} takes a form similar to that of $\SCONE$, but using polynomials of the entire Hodge Laplacian instead of considering the components $\partial_k^\top\partial_k$ and $\partial_{k+1}\partial_{k+1}^\top$ separately.
A similar argument to the proof of \cref{prop:admissibility} shows that the simplicial neural network architecture of \citet{Ebli:2020} is orientation equivariant for elementwise nonlinearity $\phi$ if and only if $\phi$ is an odd function (when $k>0$), and simplicial awareness of orders $k-1$ and $k+1$ is satisfied as well.
In particular, for a $2$-dimensional simplicial complex, simplicial awareness is satisfied when $k=1$, as in $\SCONE$.
Notably, since their architecture does not involve a ``readout'' mapping a $1$-chain to a $0$-chain, $\phi$ does not have to be nonlinear, unlike the case of $\SCONE$.
We would like to remark that the original authors used the ``Leaky~ReLU'' activation function in their empirical studies, which is not odd, thus failing to satisfy orientation equivariance.

The simplicial $2$-complex convolutional neural network of \citet{Bunch:2020} also employs convolutional layers based on the boundary maps, albeit using normalized versions of said maps following \citet{Schaub:2020}.
Rather than restricting their internal states $\bc^\ell$ to be supported on a single level of the simplicial complex as in \citet{Ebli:2020} and $\SCONE$, they propose an architecture that maintains representations on all levels of the simplicial complex.
Ignoring the details of normalizing the boundary maps and Hodge Laplacians, their convolutional layer takes the form
\begin{align}
  \bc_0^{\ell+1}&\gets\phi\left(\partial_1\bc_1^\ell\bW_\ell^{0,1}+\Delta_0\bc_0^\ell\bW_\ell^{0,0}\right)\\
  \bc_1^{\ell+1}&\gets\phi\left(\partial_2\bc_2^\ell\bW_\ell^{1,2}+\Delta_1\bc_1^\ell\bW_\ell^{1,1}+\partial_1^\top\bc_0^\ell\bW_\ell^{1,0}\right)\\
  \bc_2^{\ell+1}&\gets\phi\left(\Delta_2\bc_2^\ell\bW_\ell^{2,2}+\partial_2^\top\bc_1^\ell\bW_\ell^{2,1}\right),
\end{align}
where $\bc_0^\ell,\bc_1^\ell,\bc_2^\ell$ are the input $0,1,2$-chains, respectively, and $\phi$ is an elementwise activation function.
We illustrate the structure of this architecture in \cref{fig:bunch}.
Again, a convolutional neural network composed of such layers follows a result similar to \cref{prop:admissibility}, where it is admissible if and only if $\phi$ is odd.
If one considers the chains stored at all levels of the simplicial complex, $\phi$ does not need to be nonlinear, unlike $\SCONE$.
However, if one only takes \emph{one} of the levels of the chain complex as output, similar to the output of $\SCONE$ or the architecture of \citet{Ebli:2020}, the requirement of $\phi$ being nonlinear comes into effect.

\section{Chains, Flows, and the Hodge Decomposition}\label{app:hodge}

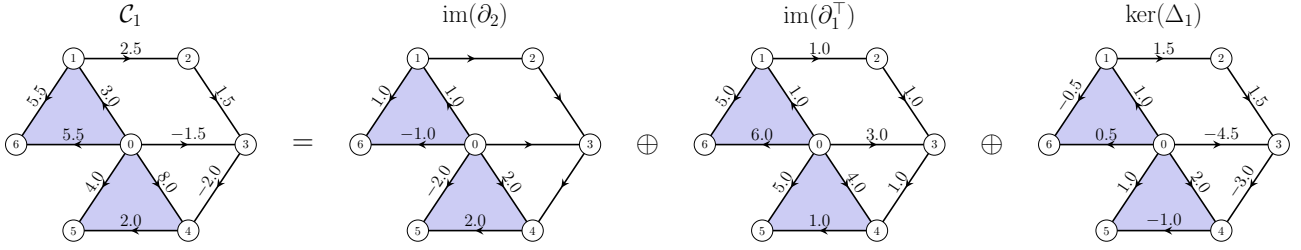
\begin{figure*}[tb!]
    \centering
    \resizebox{\linewidth}{!}{\input{figs/flow.tikz.tex}}
    \caption{
    Illustration of the Hodge Decomposition for an example $1$-chain.
    The subspace $\image(\partial_2)$ consists of $1$-chains that are curly around $2$-simplices, while $\image(\partial_1^\top)$ is determined by the differences between nodal coefficients in a $0$-chain.
    $\kernel(\Delta_1)$ corresponds to $1$-chains that are not curly, and not determined by coefficients in a $0$-chain, unlike $\image(\partial_1^\top)$.
    Figure adapted from \cite{Schaub:2020}.
    }
    \label{fig:decomposition}
\end{figure*}

We elaborate here on using $1$-chains to model flows on the edges of a simplicial complex, viewed through the lens of the Hodge Decomposition~\cf{\cref{thm:hodge}}.
Suppose we have a simplicial complex $\cX$, on which there is a \emph{flow} over the edges.
We interpret a flow as having a magnitude and an orientation.
Analogously to an electrical circuit, the magnitude of the flow is the absolute current flowing through a wire, and the orientation is determined by the sign of the measurement as well as the direction it is being measured in: that is, if the measurement direction is reversed, the sign of the measurement will change.
This skew-symmetric property is reflected by the vector space of $1$-chains over the real numbers, since for any oriented edge $[i,j]$, we have $[i,j]=-[j,i]$.
Indeed, this behavior mirrors that of a directional derivative on a surface: reversing the direction of the derivative of a function merely changes the sign.
Thus, we will use the term ``$1$-chain'' and ``flow'' interchangeably.

With this model in mind, we now provide an interpretation of \cref{thm:hodge}.
First, it will be useful to understand what \emph{integrating over a path} means in this context.
Let $S=\{i_\ell\}_{\ell=1}^m$ be a sequence of nodes, such that any two nodes that are adjacent in the sequence are also adjacent in $\cX$: we call such a sequence a \emph{path}.
If $S$ is such that the final node in the sequence is the same as the first node, we call $S$ a \emph{closed path}.
Now, let $\bc_1\in\cC_1$ be a $1$-chain, and let $\bc_S\in\cC_1$ represent $S$ as a $1$-chain in the following way:
\begin{equation}
    \bc_S = \sum_{\ell=1}^{m-1}[i_j,i_{j+1}].
\end{equation}
Then, we say that the integral of $\bc_1$ over the path $S$ is the inner product
\begin{equation}
    \langle\bc_1,\bc_S\rangle = \sum_{\ell=1}^{m-1}\langle\bc_1,[i_\ell,i_{\ell+1}]\rangle.
\end{equation}

We call the subspace $\image(\partial_2)$ \emph{curly}, since it corresponds to flows around $2$-simplices.
As pictured in~\cref{fig:decomposition}, these flows are supported strictly on the boundary of $2$-simplices, dictated by a curl about each $2$-simplex.
In particular, the pictured flow has a curl of $1.0$ about the $2$-simplex $[0,1,6]$, and a curl of $2.0$ about $[0,4,5]$.
Elsewhere, the flow takes value $0$.

The subspace $\image(\partial_1^\top)$ is referred to as \emph{gradient}, since it corresponds to flows induced by differences between so-called ``potentials'' at each node.
That is, for each $\bc_{grad}\in\image(\partial_1^\top)$ and oriented simplex $[i,j]$, there exists a $\bc_0\in\cC_0$ such that
\begin{equation}
    \langle\bc_{grad},[i,j]\rangle = \langle\bc_0,[j]\rangle - \langle\bc_0,[i]\rangle.
\end{equation}
In~\cref{fig:decomposition}, each node has a potential corresponding to its label, \eg{} node $3$ satisfies $\langle\bc_0,[3]\rangle=3.0$, so that $\langle\bc_{grad},[0,3]\rangle=3.0$.
The integral of a gradient flow, analogous to the line integral of a vector field that is the gradient of a scalar field, is \emph{path independent}, in that the integral over a path of a gradient flow only depends on the starting and end points.
That is, if $S$ and $S'$ are both paths whose initial and terminal points are the same, then for any gradient flow $\bc_{grad}\in\image(\partial_1^\top)$,
\begin{equation}
    \langle\bc_{grad},\bc_{S}\rangle=\langle\bc_{grad},\bc_{S'}\rangle.
\end{equation}
In particular, the integral of a gradient flow over a closed loop $S''$ is null: $\langle\bc_{grad},\bc_{S''}\rangle=0$.

Finally, the subspace $\kernel(\Delta_1)$ consists of \emph{harmonic} flows.
Harmonic flows satisfy two properties: the integral of a harmonic flow around a $2$-simplex is null, and the sum of the flows incident to any node is null, as exemplified in~\cref{fig:decomposition}.
More precisely, for any $\bc_{harm}\in\kernel(\Delta_1),[i]\in\cC_0, [j_0,j_1,j_2]\in\cC_2$,
\begin{align}
    \langle\bc_{harm},\partial_2[j_0,j_1,j_2]\rangle&=0 \\
    \langle\partial_1\bc_{harm},[i]\rangle&=0.
\end{align}

\section{Hodge Laplacians for Cubical $2$-Complexes}
\label{app:cubism}

In the same way that an abstract simplicial complex can be thought of as a set of points, edges, triangles, tetrahedra, etc., with the property of being closed under restriction, a cubical complex is a natural analog constructed from points, edges, squares, cubes, hypercubes, and so on.
Cubical complexes arise when considering grid-structured domains~\citep{Wagner:2012}, and particularly in \cref{sec:experiments} when considering the Berlin map data.
Defining appropriate boundary maps for cubical complexes of high dimension can be tedious, so we restrict our discussion to cubical complexes of dimension $2$, or cubical $2$-complexes.

We adapt the definition of \citet{Farley:2003} to more naturally capture the notion of orientation and boundary.
The \emph{standard abstract $k$-cube} is the set $\{0,1\}^k$.
By convention, we say $\{0,1\}^0=\{0\}$.
The \emph{faces} of the standard abstract $k$-cube are the sets taking the form $\prod_{j=1}^k A_j$, where each $A_j$ is a nonempty subset of $\{0,1\}^k$ with $|A_j|=1$ for exactly one index $j$.
By convention, we say that $\{0\}$ is a face of itself.
For a finite set $\sigma$ paired with a bijection $\psi_\sigma:\sigma\to\{0,1\}^k$, we say that a subset $\sigma'\subseteq\sigma$ is a \emph{face of $\sigma$ with respect to $\psi_\sigma$} if the image of $\sigma'$ under $\psi_\sigma$ is a face of $\{0,1\}^k$.

A \emph{cubical $2$-complex} $\cX$ over a set $\cV$ is a multiset of subsets of $\cV$ paired with a set of bijections $\{\psi_\sigma:\sigma\to\{0,1\}^{k_\sigma}\}_{\sigma\in\cX}$, where $k_\sigma\in\{0,1,2\}$ for each $\sigma\in\cX$, with the following properties:
\begin{enumerate}
\item $\cX$ covers $\cV$.
\item For each $\sigma\in\cX$, if $\sigma'\subseteq\sigma$, then $\sigma'\in\cX$ if and only if $\sigma'$ is a face of $\sigma$ with respect to $\psi_\sigma$.
\end{enumerate}
We denote the set of elements of $\cX$ with $2^k$ elements as $\cX_k$: the elements of $\cX_k$ are naturally referred to as $k$-cubes, due to the existence of a bijection with the standard abstract $k$-cube.

Note the similarities and differences with the case of an abstract simplicial complex: for a simplex in an abstract simplicial complex, all of its subsets are faces, so that all of its faces are contained in the abstract simplicial complex due to closure under restriction.
In the case of an abstract cubical complex, we only require closure under restriction \emph{to faces} with respect to the bijections $\psi_\sigma$.
This is more general than closure under restriction, and in fact subsumes closure under restriction for abstract simplicial complexes.
We illustrate this in \cref{fig:cubes}.

In order to define a boundary operator, we first define an appropriate notion of orientation.
For $0$-cubes and $1$-cubes, an orientation is the exact same as in the case of a simplicial complex.
That is, if $\sigma=\{i_0\}$ for some $i_0\in\cV$, the only orientation of $\sigma$ is $[i_0]$.
Similarly, if $\sigma=\{i_0,i_1\}$, there are two orientations of $\sigma$: $[i_0,i_1]$ and $[i_1,i_0]$.
When $\sigma\in\cX_2$, extra restrictions on the notion of orientations are needed.
Suppose $\sigma=\{i_0,i_1,i_2,i_3\}$ is an element of $\cX_2$.
An ordered sequence of the elements of $\sigma$ is said to be an orientation with respect to $\psi_\sigma$ if each pair of cyclically adjacent elements in the sequence forms a face of $\sigma$ with respect to $\psi_\sigma$.
Based on adjacency being considered in a cyclic fashion, we take these orientations modulo cyclic permutations.
By convention, we say that ``reversals'' of an orientation change the sign, so that $[i_0,i_1,i_2,i_3]=-[i_3,i_2,i_1,i_0]$.
One can check that modulo cyclic permutations, these are the only two orientations of a $2$-cube.

As before, denote by $\cC_k$ the vector space with the oriented $k$-simplices of $\cX$ as a canonical orthonormal basis, defined over the field of real numbers.
The boundary map $\partial_2:\cC_2\to\cC_1$ of an oriented $2$-cube is defined as follows:
\begin{equation}
  \partial_2([i_0,i_1,i_2,i_3]) = [i_0,i_1] + [i_1,i_2] + [i_2,i_3] - [i_0,i_3].
\end{equation}
The boundary map $\partial_1$ is defined as before:
\begin{equation}
  \partial_1([i_0,i_1]) = [i_1]-[i_0].
\end{equation}
In this setting, \cref{lem:fun-thm-alg-top} holds: $\partial_1\partial_2=0$.
Moreover, taking the adjoint of the operators $\partial_1$ and $\partial_2$, we can construct the $k^\mathrm{th}$ cubical Hodge Laplacian in the expected way:
\begin{equation}
  \Delta_k = \partial_k^\top\partial_k+\partial_{k+1}\partial_{k+1}^\top,
\end{equation}
with the corresponding cubical analog to the Hodge Decomposition:
\begin{equation}
  \cC_k = \image(\partial_{k+1})\oplus\image(\partial_k^\top)\oplus\kernel(\Delta_k).
\end{equation}

Given the representation of boundary maps acting on formal sums of oriented cubes in an abstract cubical complex, architectures analogous to $\SCONE$ follow naturally via simple substitutions of the boundary maps.

\begin{figure*}
  \centering
  \hspace*{\fill}
  \begin{minipage}[c][5cm][c]{0.4\linewidth}
    \resizebox{\linewidth}{!}{\input{figs/cube.tikz.tex}}
  \end{minipage}
  \hfill
  \begin{minipage}[c][5cm][t]{0.18\linewidth}
    \vspace{0.25cm}
    \resizebox{\linewidth}{!}{
      \begin{tabular}{c|cc}
        & $\psi_1$ & $\psi_2$ \\
        \hline
        $i_0$ & $(0,0)$ & $(0,0)$ \\
        $i_1$ & $(0,1)$ & $(0,1)$ \\
        $i_2$ & $(1,1)$ & $(1,0)$ \\
        $i_3$ & $(1,0)$ & $(1,1)$ \\
      \end{tabular}
    }
  \end{minipage}
  \hfill
  \begin{minipage}[c][5cm][t]{0.2\linewidth}
    \vspace{0.4cm}
    \resizebox{\linewidth}{!}{\input{figs/abstract-cube.tikz.tex}}
  \end{minipage}
  \hspace*{\fill}
  \caption{
    A set does not a $2$-cube make.
    The defined bijection between an element of a cubical $2$-complex and the standard $2$-cube dictates what its faces are.
    Pictured are two $2$-cubes over the same set of nodes, but with different bijections $\psi_1,\psi_2$ between the standard abstract $2$-cube.
    Although each $2$-cube is defined over the same set of vertices, the faces induced by the bijections $\psi_1,\psi_2$ differ.
  }
  \label{fig:cubes}
\end{figure*}
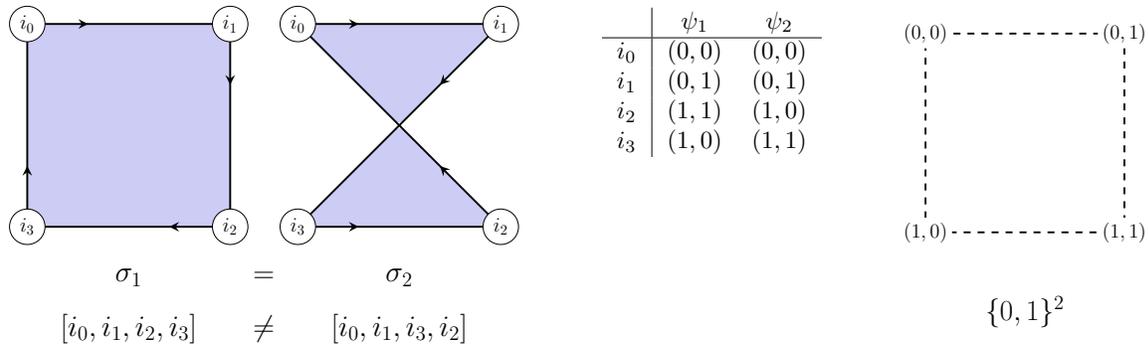

\bibliography{ref}
\bibliographystyle{icml2020}

\end{document}

%% file: figs/scone.tikz.tex
\begin{tikzpicture}[every node/.style={scale=1.2},rounded corners=3pt]
  \matrix (m) [matrix of math nodes,row sep=2em,column sep=1em,minimum width=2em]
  {
    {} & {} & \mathbf{d}_2^\ell & {} & {} & {} & {}  \\
    \mathbf{c}_1^\ell &
    {} &
    \mathbf{d}_1^\ell &
    {} &
    \mathbf{\sum} &
    \mathbf{\phi(\cdot)} &
    \mathbf{c}_1^{\ell+1} \\
    {} & {} & \mathbf{d}_0^\ell & {} & {} & {} & {} \\
  };
  
  \path[-stealth]
  (m-2-1)
  edge node [fill=white] {$\partial_2^\top$} (m-1-3)
  edge node [fill=white] {$\mathrm{Id}$} (m-2-3)
  edge node [fill=white] {$\partial_1$} (m-3-3)

  (m-1-3) edge node [fill=white] {$\partial_2$} (m-2-5)
  (m-2-3) edge node [fill=white] {$\mathrm{Id}$} (m-2-5)
  (m-3-3) edge node [fill=white] {$\partial_1^\top$} (m-2-5)
  
  (m-2-5) edge (m-2-6)
  (m-2-6) edge (m-2-7);

  \begin{scope}[on background layer]
    \filldraw[fill=blue,fill opacity=0.2]
    (m-2-1.south west) rectangle (m-2-1.north east);
    
    \filldraw[fill=green,fill opacity=0.2]
    (m-2-7.south west) rectangle (m-2-7.north east);
    
    \draw (m-3-3.south west) rectangle (m-1-3.north east);
  \end{scope}
    
\end{tikzpicture}


%% file: figs/synthetic-image-data/synthetic.tikz.tex
\begin{tikzpicture}[>=Stealth]
  \loaddata{./figs/synthetic-image-data/coords.dat}
  \foreach \i/\x/\y in \loadeddata{\coordinate (c\i) at (10*\x,9*\y);}

  \loaddata{./figs/synthetic-image-data/faces.dat}
  \foreach \i/\j/\k in \loadeddata{\filldraw[ultra thin, fill=blue!90!gray, fill opacity=0.1] (c\i) -- (c\j) -- (c\k) -- cycle;}

  \foreach \pathnum in {0,...,2} {
    \loaddata{./figs/synthetic-image-data/path-\pathnum.dat}
    \foreach \i/\j/\k in \loadeddata {
      \draw [ultra thick, red, ->] (c\i) -- (c\j);
    }
  }
\end{tikzpicture}

%% file: figs/drifters-image-data/drifters.tikz.tex
\begin{tikzpicture}[>=Stealth]
  \loaddata{./figs/drifters-image-data/coords.dat}
  \foreach \i/\x/\y in \loadeddata{\coordinate (c\i) at (30*\x,24*\y);}

  \loaddata{./figs/drifters-image-data/faces.dat}
  \foreach \i/\j/\k in \loadeddata{\filldraw[ultra thin, fill=blue!90!gray, fill opacity=0.1] (c\i) -- (c\j) -- (c\k) -- cycle;}

  \foreach \pathnum in {0,...,2} {
    \loaddata{./figs/drifters-image-data/path-\pathnum.dat}
    \foreach \i/\j/\k in \loadeddata {
      \draw [ultra thick, red, ->] (c\i) -- (c\j);
    }
  }
\end{tikzpicture}

%% file: figs/bunch.tikz.tex
\begin{tikzpicture}[every node/.style={scale=1.2},rounded corners=3pt]
  \matrix (m) [matrix of math nodes,row sep=3em,column sep=5em,minimum width=2em]
  {
    \mathbf{c}_2^0 & \mathbf{c}_2^1 & \mathbf{c}_2^2 & \mathbf{c}_2^3 \\
    \mathbf{c}_1^0 & \mathbf{c}_1^1 & \mathbf{c}_1^2 & \mathbf{c}_1^3 \\
    \mathbf{c}_0^0 & \mathbf{c}_0^1 & \mathbf{c}_0^2 & \mathbf{c}_0^3 \\
  };
  
  \path[-stealth]
  (m-1-1)
  edge (m-1-2)
  edge (m-2-2)
  
  (m-2-1)
  edge (m-1-2)
  edge (m-2-2)
  edge (m-3-2)

  (m-3-1)
  edge (m-2-2)
  edge (m-3-2)

  (m-1-2)
  edge (m-1-3)
  edge (m-2-3)

  (m-2-2)
  edge (m-1-3)
  edge (m-2-3)
  edge (m-3-3)

  (m-3-2)
  edge (m-2-3)
  edge (m-3-3)

  (m-1-3)
  edge (m-1-4)
  edge (m-2-4)

  (m-2-3)
  edge (m-1-4)
  edge (m-2-4)
  edge (m-3-4)

  (m-3-3)
  edge (m-2-4)
  edge (m-3-4);

  \begin{scope}[on background layer]
    \filldraw[fill=blue,fill opacity=0.2]
    (m-3-1.south west) rectangle (m-1-1.north east);
    
    \filldraw[fill=green,fill opacity=0.2]
    (m-3-4.south west) rectangle (m-1-4.north east);
  \end{scope}
    
\end{tikzpicture}


%% file: figs/flow.tikz.tex
\begin{tikzpicture}[every node/.style={scale=0.6},rounded corners=3pt]

  \def \bigshift {6cm}
  \def \xshift {1cm}
  \def \yshift {1.5cm}
  
  \foreach \x/\y/\n in {0/0/0, -1/1/1, 1/1/2, 2/0/3, 1/-1/4, -1/-1/5, -2/0/6}
  {
    \foreach \type/\s in {flow/0,curl/1,grad/2,harm/3}
    {
      \coordinate (\type\n) at (\x*\xshift+\s*\bigshift,\y*\yshift);
    }
  }
  
  \begin{scope}[thick,decoration={
    markings,
    mark=at position 0.5 with {\arrow{stealth}}}
    ]
    \foreach \type in {flow,curl,grad,harm}
    {
      \draw[postaction={decorate}] (\type0)--(\type1);
      \draw[postaction={decorate}] (\type0)--(\type3);
      \draw[postaction={decorate}] (\type0)--(\type4);
      \draw[postaction={decorate}] (\type0)--(\type5);
      \draw[postaction={decorate}] (\type0)--(\type6);
      \draw[postaction={decorate}] (\type1)--(\type2);
      \draw[postaction={decorate}] (\type1)--(\type6);
      \draw[postaction={decorate}] (\type2)--(\type3);
      \draw[postaction={decorate}] (\type3)--(\type4);
      \draw[postaction={decorate}] (\type4)--(\type5);
    }
  \end{scope}
  
  \foreach \n in {0,...,6}
  {
    \foreach \type in {flow,curl,grad,harm}
    {
      \node[draw, circle, fill=white] at (\type\n) {$\n$};
    }
  }
  
  \begin{scope}[on background layer]
    \foreach \type in {flow,curl,grad,harm}
    {
      \fill[fill=blue!80!black, fill opacity=0.2] (\type0) -- (\type1) -- (\type6) -- cycle;
      \fill[fill=blue!80!black, fill opacity=0.2] (\type0) -- (\type4) -- (\type5) -- cycle;
    }
  \end{scope}
  
  \node at (0.5*\bigshift,0) {\Huge$=$};
  \node at (1.5*\bigshift,0) {\Huge$\oplus$};
  \node at (2.5*\bigshift,0) {\Huge$\oplus$};
  
  \node at (0,1.5*\yshift) {\huge$\cC_1$};
  \node at (\bigshift,1.5*\yshift) {\huge$\image(\partial_2)$};
  \node at (2*\bigshift,1.5*\yshift) {\huge$\image(\partial_1^\top)$};
  \node at (3*\bigshift,1.5*\yshift) {\huge$\kernel(\Delta_1)$};
  
  
  \draw[draw=none] (curl0) -- node[above, sloped] {\Large$1.0$} (curl1);
  \draw[draw=none] (curl1) -- node[above, sloped] {\Large$1.0$} (curl6);
  \draw[draw=none] (curl0) -- node[above, sloped] {\Large$-1.0$} (curl6);
  
  \draw[draw=none] (curl0) -- node[above, sloped] {\Large$2.0$} (curl4);
  \draw[draw=none] (curl4) -- node[above, sloped] {\Large$2.0$} (curl5);
  \draw[draw=none] (curl0) -- node[above, sloped] {\Large$-2.0$} (curl5);
  
  \draw[draw=none] (grad0) -- node[above, sloped] {\Large$1.0$} (grad1);
  \draw[draw=none] (grad0) -- node[above, sloped] {\Large$3.0$} (grad3);
  \draw[draw=none] (grad0) -- node[above, sloped] {\Large$4.0$} (grad4);
  \draw[draw=none] (grad0) -- node[above, sloped] {\Large$5.0$} (grad5);
  \draw[draw=none] (grad0) -- node[above, sloped] {\Large$6.0$} (grad6);
  \draw[draw=none] (grad1) -- node[above, sloped] {\Large$1.0$} (grad2);
  \draw[draw=none] (grad1) -- node[above, sloped] {\Large$5.0$} (grad6);
  \draw[draw=none] (grad2) -- node[above, sloped] {\Large$1.0$} (grad3);
  \draw[draw=none] (grad3) -- node[above, sloped] {\Large$1.0$} (grad4);
  \draw[draw=none] (grad4) -- node[above, sloped] {\Large$1.0$} (grad5);
  
  \draw[draw=none] (harm0) -- node[above, sloped] {\Large$1.0$} (harm1);
  \draw[draw=none] (harm0) -- node[above, sloped] {\Large$-4.5$} (harm3);
  \draw[draw=none] (harm0) -- node[above, sloped] {\Large$2.0$} (harm4);
  \draw[draw=none] (harm0) -- node[above, sloped] {\Large$1.0$} (harm5);
  \draw[draw=none] (harm0) -- node[above, sloped] {\Large$0.5$} (harm6);
  \draw[draw=none] (harm1) -- node[above, sloped] {\Large$1.5$} (harm2);
  \draw[draw=none] (harm1) -- node[above, sloped] {\Large$-0.5$} (harm6);
  \draw[draw=none] (harm2) -- node[above, sloped] {\Large$1.5$} (harm3);
  \draw[draw=none] (harm3) -- node[above, sloped] {\Large$-3.0$} (harm4);
  \draw[draw=none] (harm4) -- node[above, sloped] {\Large$-1.0$} (harm5);
  
  \draw[draw=none] (flow0) -- node[above, sloped] {\Large$3.0$} (flow1);
  \draw[draw=none] (flow0) -- node[above, sloped] {\Large$-1.5$} (flow3);
  \draw[draw=none] (flow0) -- node[above, sloped] {\Large$8.0$} (flow4);
  \draw[draw=none] (flow0) -- node[above, sloped] {\Large$4.0$} (flow5);
  \draw[draw=none] (flow0) -- node[above, sloped] {\Large$5.5$} (flow6);
  \draw[draw=none] (flow1) -- node[above, sloped] {\Large$2.5$} (flow2);
  \draw[draw=none] (flow1) -- node[above, sloped] {\Large$5.5$} (flow6);
  \draw[draw=none] (flow2) -- node[above, sloped] {\Large$1.5$} (flow3);
  \draw[draw=none] (flow3) -- node[above, sloped] {\Large$-2.0$} (flow4);
  \draw[draw=none] (flow4) -- node[above, sloped] {\Large$2.0$} (flow5);
  
\end{tikzpicture}


%% file: figs/cube.tikz.tex
\begin{tikzpicture}[every node/.style={scale=0.6},rounded corners=3pt]

  \def \bigshift {2cm}
  \def \xshift {1.5cm}
  \def \yshift {1.5cm}
  
  \foreach \x/\y/\n in {-1/1/0, 1/1/1, 1/-1/2, -1/-1/3}
  {
    \foreach \type/\s in {loop/-1,cross/1}
    {
      \coordinate (\type\n) at (\s*\bigshift+\x*\xshift,\y*\yshift);
    }
  }

  \coordinate (crosscenter) at (\bigshift,0);
  
  \begin{scope}[thick,decoration={
    markings,
    mark=at position 0.3 with {\arrow{stealth}}}
    ]
    \draw[postaction={decorate}] (loop0)--(loop1);
    \draw[postaction={decorate}] (loop1)--(loop2);
    \draw[postaction={decorate}] (loop2)--(loop3);
    \draw[postaction={decorate}] (loop3)--(loop0);

    \draw[postaction={decorate}] (cross0)--(cross1);
    \draw[postaction={decorate}] (cross1)--(cross3);
    \draw[postaction={decorate}] (cross3)--(cross2);
    \draw[postaction={decorate}] (cross2)--(cross0);
  \end{scope}
  
  \foreach \n in {0,...,3}
  {
    \foreach \type in {loop,cross}
    {
      \node[draw, circle, fill=white] at (\type\n) {\Large$i_\n$};
    }
  }
  
  \begin{scope}[on background layer]
    \fill[fill=blue!80!black, fill opacity=0.2] (loop0) -- (loop1) -- (loop2) -- (loop3) -- cycle;
    \fill[fill=blue!80!black, fill opacity=0.2] (cross0) -- (cross1) -- (crosscenter) -- cycle;
    \fill[fill=blue!80!black, fill opacity=0.2] (cross2) -- (cross3) -- (crosscenter) -- cycle;
  \end{scope}
  
  \node at (-\bigshift,-1.5*\yshift) {\huge$\sigma_1$};
  \node at (0,-1.5*\yshift) {\huge$=$};
  \node at (\bigshift,-1.5*\yshift) {\huge$\sigma_2$};

  \node at (-\bigshift,-2*\yshift) {\huge$[i_0,i_1,i_2,i_3]$};
  \node at (0,-2*\yshift) {\huge$\neq$};
  \node at (\bigshift,-2*\yshift) {\huge$[i_0,i_1,i_3,i_2]$};
  
\end{tikzpicture}


%% file: figs/abstract-cube.tikz.tex
\begin{tikzpicture}[every node/.style={scale=0.6},rounded corners=3pt]

  \def \xshift {1.5cm}
  \def \yshift {1.5cm}
  
  \foreach \x/\y/\n in {-1/1/0, 1/1/1, 1/-1/2, -1/-1/3}
  {
    \coordinate (cube\n) at (\x*\xshift,\y*\yshift);
  }

  \begin{scope}[thick,dashed]
    \draw[] (cube0)--(cube1);
    \draw[] (cube1)--(cube2);
    \draw[] (cube2)--(cube3);
    \draw[] (cube3)--(cube0);
  \end{scope}
  
  \node[fill=white] at (cube0) {\Large$(0,0)$};
  \node[fill=white] at (cube1) {\Large$(0,1)$};
  \node[fill=white] at (cube2) {\Large$(1,1)$};
  \node[fill=white] at (cube3) {\Large$(1,0)$};
  
  \node at (0,-1.8*\yshift) {\huge$\{0,1\}^2$};
  
\end{tikzpicture}
